%% file: main.tex
\documentclass[10pt]{article} 
\usepackage[accepted]{rlc}

\usepackage{amssymb}            
\usepackage{mathtools}          
\usepackage{mathrsfs}           
\mathtoolsset{showonlyrefs}     
\usepackage{graphicx}           
\usepackage[space]{grffile}     
\usepackage{url}                

\input{math_commands.tex}

\input{commands}

\title{Causal Contextual Bandits with Adaptive Context}

\author{Rahul Madhavan,\\
        \addr IISc Bangalore
        \And
        Aurghya Maiti,\\
        \addr Columbia University
        \And
        Gaurav Sinha,\\
        \addr Microsoft Research
        \And
        Siddharth Barman,\\
        \addr IISc Bangalore}
    
\begin{document}

\maketitle

\begin{abstract}
  We study a variant of causal contextual bandits where the context is chosen based on an initial intervention chosen by the learner. At the beginning of each round, the learner selects an initial action, depending on which a stochastic context is revealed by the environment. Following this, the learner then selects a final action and receives a reward. Given $T$ rounds of interactions with the environment, the objective of the learner is to learn a policy (of selecting the initial and the final action) with maximum expected reward. In this paper we study the specific situation where every action corresponds to intervening on a node in some known causal graph. We extend prior work from the deterministic context setting to obtain simple regret minimization guarantees. This is achieved through an instance-dependent causal parameter, $\lambda$, which characterizes our upper bound. Furthermore, we prove that our simple regret is essentially tight for a large class of instances. A key feature of our work is that we use convex optimization to address the bandit exploration problem. We also conduct experiments to validate our theoretical results, and release our code at the \href{https://github.com/adaptiveContextualCausalBandits/aCCB}{project GitHub Repository}.
\end{abstract}

\input{Includes/introduction}
\input{Includes/notations}

\input{Includes/algorithm}
\input{Includes/lower-bound}

\input{Includes/experiments}
\input{Includes/conclusion}

\input{Includes/acknowledgements}


\bibliography{References}
\bibliographystyle{rlc}

\appendix

\vspace{0.15in}
\input{appendix/appendixRelatedWork.tex}

\input{appendix/appendixAlgorithmsInDetail.tex}

\input{appendix/appendixProofOfUpperbound.tex}
\input{appendix/appendixBoundingProbabilities.tex}

\input{appendix/appendixConvex.tex}
\input{appendix/appendixLowerBound.tex}

\end{document}

%% file: math_commands.tex
\usepackage{amsmath,amsfonts,bm}

\def\eqref#1{equation~\ref{#1}}

\def\1{\bm{1}}

\DeclareMathAlphabet{\mathsfit}{\encodingdefault}{\sfdefault}{m}{sl}
\SetMathAlphabet{\mathsfit}{bold}{\encodingdefault}{\sfdefault}{bx}{n}

\newcommand{\E}{\mathbb{E}}

\newcommand{\R}{\mathbb{R}}

\DeclareMathOperator*{\argmax}{arg\,max}
\DeclareMathOperator*{\argmin}{arg\,min}

%% file: commands.tex
\definecolor{DarkBlue}{rgb}{0.1,0.1,0.5}
\definecolor{DarkGreen}{rgb}{0.1,0.5,0.1}
\hypersetup{
	colorlinks   = true,
	linkcolor    = DarkBlue, 
	urlcolor     = DarkBlue, 
	citecolor    = DarkGreen 
}

\usepackage{xfrac, amsthm,thm-restate,thmtools}
\usepackage{thmtools,thm-restate}
\usepackage{amsfonts,dsfont}
\usepackage[mathscr]{euscript}
\usepackage{booktabs} 

\usepackage{svg}

\definecolor{DarkBlue}{rgb}{0.3,0.3,0.70}
\definecolor{azure}{rgb}{0.0, 0.5, 1.0}
\definecolor{darkcerulean}{rgb}{0.03, 0.27, 0.49}
\definecolor{denim}{rgb}{0.08, 0.38, 0.74}
\definecolor{DarkGreen}{rgb}{0.3,0.7,0.3}
\definecolor{lighter-gray}{gray}{0.95}
\hypersetup{
	colorlinks   = true,
	linkcolor    = DarkBlue, 
	urlcolor     = DarkBlue, 
	citecolor    = DarkGreen 
}
\definecolor{AlgHighlight}{HTML}{228B22}

\newtheoremstyle{thmstyle}
{0.5em} 
{0.15em} 
{} 
{} 
{\bfseries} 
{.} 
{.5em} 
{} 

\theoremstyle{thmstyle} 
\newtheorem{thm}{Theorem}
\newtheorem{lem}{Lemma}
\newtheorem{propn}{Proposition}[section]
\newtheorem{claim}{Claim}[section]
\newtheorem{cor}{Corollary}
\newtheorem*{nonumtheorem}{Theorem}

\theoremstyle{definition}
\newtheorem{defn}{Definition}

\theoremstyle{remark}

\renewcommand{\1}{ \mathds{1}}
\newcommand{\A}{\mathcal{A}}
\renewcommand{\E}{\mathbb{E}}

\newcommand{\I}{\mathcal{I}}
\newcommand{\J}{\mathcal{J}}

\newcommand{\N}{\mathbb{N}}

\renewcommand{\O}{\mathcal{O}}
\renewcommand{\P}{\mathcal{P}}

\newcommand{\prob}{\mathbb{P}}

\renewcommand{\R}{\mathbb{R}}
\renewcommand{\S}{\mathcal{S}}

\newcommand{\Regret}{{\rm Regret}}

\newcommand{\CE}{\textsc{ConvExplore}}
\newcommand{\UE}{\textsc{UnifExplore}}

\newcommand{\SRM}{\textsc{SRM-ALG}}

\renewcommand{\I}{\mathcal{A}}
\renewcommand{\S}{\mathcal{C}}
\newcommand{\C}{\mathcal{C}}

\newcommand{\tr}{\top}

\newcommand{\CommentLines}[1]{}

\usepackage{array}
\newcolumntype{x}[1]{>{\centering\let\newline\\\arraybackslash\hspace{0pt}}m{#1}}
\usepackage{subfigure}
\usepackage{colortbl}

\usepackage{algorithm,algorithmicx}
\usepackage{wrapfig}
\usepackage{svg}

%% file: Includes/introduction.tex
\section{Introduction}
	
	

Recent years have seen an active interest in causal bandits from the research community \citep{Lattimore,sen2017identifying,sen2017contextual,LeeBarenBoim2018,yabe2018causal,  LeeBareinboim2019,lu2020regret,Gaurav2020,lu2021causal,lu2022efficient,maiti2022causal,varici2022causal,subramanian2022causal,xiong2023combinatorial}. 
In this setting, one assumes an environment comprising of causal variables that are random variables that influence each other as per a given causal (directed, and acyclic) graph. 
Specifically, the edges in the causal DAG represent causal relationships between variables in the environment. 
If one of these variables is designated as a reward variable, then the goal of a learner then is to maximize their reward by \emph{intervening} on certain variables (i.e., by fixing the values of certain variables). The rest of the variables, that are not intervened upon, take values as per their conditional distributions, given their parents in the causal graph. In this work, as is common in literature, we assume that the variables take values in $\{0,1\}$. 
Of particular interest are causal settings wherein the learner is allowed to perform \emph{atomic interventions}. Here, at most one causal variable can be set to a particular value, while other variables take values in accordance with their underlying distributions. 


It is relevant to note that when a learner performs an intervention in a causal graph, they get to observe the values of multiple other  variables in the causal graph. Hence,  the collective dependence of the reward on the variables is observed through each intervention. That is, from such an observation, the learner may be able to make inferences about the (expected) reward under other values for the causal variables \citep{PetersBook}. In essence, with a single intervention, the learner is allowed to intervene on a variable (in the causal graph), allowed to observe all other variables, and further, is privy to the effects of such an intervention. Indeed, such an observation in a causal graph is richer than a usual sample from a stochastic process. Hence, a standard goal in causal bandits is to understand the power and limitations of interventions. This goal manifests in the form of developing algorithms that identify intervention(s) that lead to high rewards, while using as few observations/interventions as possible. We use the term \emph{intervention complexity} (rather than sample complexity) for our algorithm, to emphasize that interventions are richer than samples.
\begin{figure*}
\centering
\label{figure: advertiser motivation}
\vspace{-0.3in}
\includegraphics[width=0.98\linewidth]{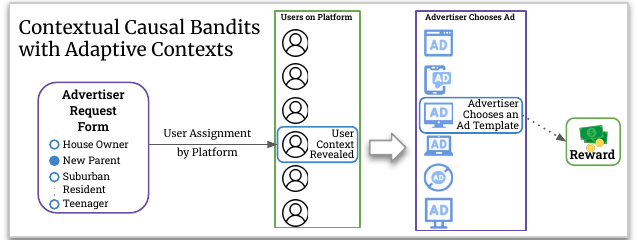}
\caption{Flowchart illustrating the decision-making process of an advertiser posting ads on a platform like Amazon, and the subsequent interaction with the platform.}
\vspace{-0.2in}
\end{figure*}

 In the learning literature, there are several objectives that an algorithm designer might consider. Cumulative regret, simple regret, and average regret have prominently been studied in literature \citep{LattimoreBook,slivkins2019introduction}.
 In this work we focus on minimizing simple regret, wherein the algorithm is given a time budget, up to which it may explore, at which time it has to output a near-optimal policy. 

Addressing causal bandits, the notable work of \cite{Lattimore} obtains an intervention-complexity bound for minimizing simple regret with a focus on atomic interventions and parallel causal graphs. \citet{maiti2022causal} extend this work to obtain intervention-complexity bounds for simple regret in causal graphs with unobserved variables. The work by \citet{lu2022efficient} extends this setting to causal Markov decision processes (MDPs), while addressing the cumulative regret objective. Combinatorial causal bandits have been studied by \citet{Feng_Chen_2023} and \cite{xiong2023combinatorial}. 


Causal contextual bandits have been studied by \citet{subramanian2022causal} where the contexts may be chosen by the learner (rather than be provided by the environment). Here we generalize \citet{subramanian2022causal} to a setting where the context is provided by the environment, adaptively, in response to an initial choice of the learner.



\textbf{Motivating Example:} Consider an advertiser looking to post ads on a web-page, say Amazon. They may make requests for a certain type of user demographic to Amazon. Based on this initial request, the platform may actually choose one particular user to show the ad to. At this time, certain details about the user are revealed to the advertiser. For example, the platform may reveal some of the user demographics, as well as certain details about their device. Based on these details, the advertiser may choose one particular ad to show the user. In case the user clicks the ad, the advertiser receives a reward. The goal of the learner is to find optimal choices for initial user preference, as well as ad-content such that user clicks are maximized. We illustrate this example through \hyperref[figure: advertiser motivation]{Figure \ref{figure: advertiser motivation}} where we indicate the choices available for template and content interventions.

	\subsection{Our Contributions}
	
	We develop an algorithm to identify near-optimal interventions in causal bandits with adaptive context, and show that the simple regret of such an algorithm is indeed tight for several instances. 
    We highlight the main contributions of our work below.

    \textbf{1.} We develop and analyze an algorithm for minimizing simple regret for causal bandits with adaptive context in an intervention efficient manner. We provide an upper-bound on intervention complexity in Theorem \ref{theorem:main}.
    
    \textbf{2.} Interestingly, the intervention complexity of our algorithm depends on an instance dependent structural parameter---referred to as $\lambda$ (see equation (\ref{eqn:lambda}))--- which may be much lower than $nk$, where $n$ is the number of interventions and $k$ is the number of contexts. 

    \textbf{3.} Notably, our algorithm uses a convex program to identify optimal interventions. Unlike prior work that uses optimization to design exploration (for example see \citet{yabe2018causal}), we show (in Appendix Section \ref{section: nature of optimization problems}) that the optimization problem we design is convex, and is thus computationally efficient. Using convex optimization to design efficient exploration is in fact a distinguishing feature of our work.

    \textbf{4.} We provide lower bound guarantees showing that our regret guarantee is tight (up to a log factor) for a large family of instances (see Section \ref{section: lower bounds} and Appendix Section \ref{appendix section: analysis of lower bounds}).

    \textbf{5.} We demonstrate using experiments (see Section \ref{section: experiments}) that our algorithm performs exceeding well as compared to other baselines. We note that this is because $\lambda \ll nk$ for $n$ causal variables and $k$ contexts.

    In conclusion, we provide a novel \hyperref[alg:best policy generator]{convex-optimization based algorithm} for Causal MDP exploration. We analyze the algorithm to come up with an \hyperref[eqn:lambda]{instance dependent parameter $\lambda$}. Further, we prove that our algorithm is sample efficient (see \hyperref[theorem:main]{Theorems \ref{theorem:main}} and \hyperref[theorem: Lower Bound for our algorithm]{\ref{theorem: Lower Bound for our algorithm}}).

\vspace{-0.05in}
\subsection{Additional Related Work}
\label{sec: related work}

\vspace{-0.05in}
\begin{table}[!h]
    \footnotesize
    \renewcommand{\arraystretch}{1.1}
    \begin{center}
    \begin{tabular}{|x{9.7cm} | x{4.5cm} |}
    \hline
    \textbf{Description} & \textbf{Reference}\\
    \hline
    Simple regret for bandits with parallel causal graphs & \citet{Lattimore}\\
    \arrayrulecolor{lighter-gray}
    \hline
    Simple regret for atomic soft interventions &\cite{sen2017identifying}\\
    \hline
    Simple regret for non-atomic interventions in causal bandits &\cite{yabe2018causal}\\
    \hline
    Cumulative regret for general causal graphs & \citet{lu2020regret}\\
    \hline
    Simple regret in the presence of unobserved confounders & \citet{maiti2022causal}\\
    \hline
    Cumulative regret for unknown causal graph structure & \citet{lu2021causal}\\
    \hline
    Cumulative regret for causal contextual bandits with latent confounders& \citet{sen2017contextual}\\
    \hline
    Simple and cumulative regret for budgeted causal bandits & \citet{Gaurav2020}\\
    \hline
    Cumulative regret for Linear SEMs & \citet{varici2022causal}\\
    \hline
        Cumulative regret for combinatorial causal bandits & \citet{Feng_Chen_2023}\\
    \hline
    Cumulative regret for Causal MDPs & \citet{lu2022efficient}\\
    \hline
    Best-intervention for combinatorial causal bandits & \citet{xiong2023combinatorial}\\
    \hline
    Additive Causal Bandits with Unknown Graph & \citet{malek2023additive}\\
    \hline
    Structural Causal Bandits with Unobserved Confounders &
    \cite{wei2024approximate}\\
    \hline
    Confounded Budgeted Causal Bandits &
    \cite{jamshidi2024confounded}\\
    \hline
    Cumulative Regret for Causal Bandits with Lipschitz SEMs &
    \cite{yan2024causal} \\
    \hline
    Simple regret for causal contextual bandits & \citet{subramanian2022causal}\\
    \hline
    Simple regret for causal contextual bandits with adaptive context & \textbf{Our work} \\
    \arrayrulecolor{black}
    \hline\hline
    \end{tabular}
    \renewcommand{\arraystretch}{1}
    \caption{Summary of prior work in causal bandits}
    \label{table: summary of prior work in causal bandits}
    \end{center}
\end{table}

\vspace{-0.08in}
Ever since the introduction of the causal bandit framework by \citet{Lattimore}, we have seen multiple works address causal bandits in various degrees of generality and using different modelling assumptions. \cite{sen2017identifying} addressed the issue of soft atomic interventions using an importance sampling based approach. Soft interventions in the linear structural equation model (SEM) setting was addressed recently by \citet{varici2022causal}. \citet{yabe2018causal} proposed an optimization based approach for non-atomic interventions. This work was extended by \citet{xiong2023combinatorial} to provide instance dependent regret bounds. They also provide guarantees for binary generalized linear models (BGLMs). The question of unknown causal graph structure was addressed by \citet{lu2021causal}, whereas \citet{Gaurav2020} study the case where interventions are more expensive than observations.

\citet{maiti2022causal} addressed simple regret for graphs containing hidden confounding causal variables, while cumulative regret in general causal graphs was addressed by \citet{lu2020regret}. A notable work by \citet{lu2022efficient} formulates the framework for causal MDPs, and they provide cumulative regret guarantees in this setting. Causal contextual bandits were addressed by \citet{subramanian2022causal,sen2017contextual}, and we extend these works to adaptive contexts. 

We summarize the main works in this thread in \hyperref[table: summary of prior work in causal bandits]{Table \ref{table: summary of prior work in causal bandits}} and provide a more detailed set of related works in \hyperref[appendix sec: related work appendix]{Appendix \ref{appendix sec: related work appendix}}.

%% file: Includes/notations.tex
\section{Notations and Preliminaries}
\vspace{-0.1in}
\label{section: notation and preliminaries}
\begin{figure*}    \centering     
    \subfigure[Illustrative figure for causal contextual bandit with adaptive context.]{
    \label{figure: two-layer MDP}
    \includegraphics[width=0.47\textwidth]{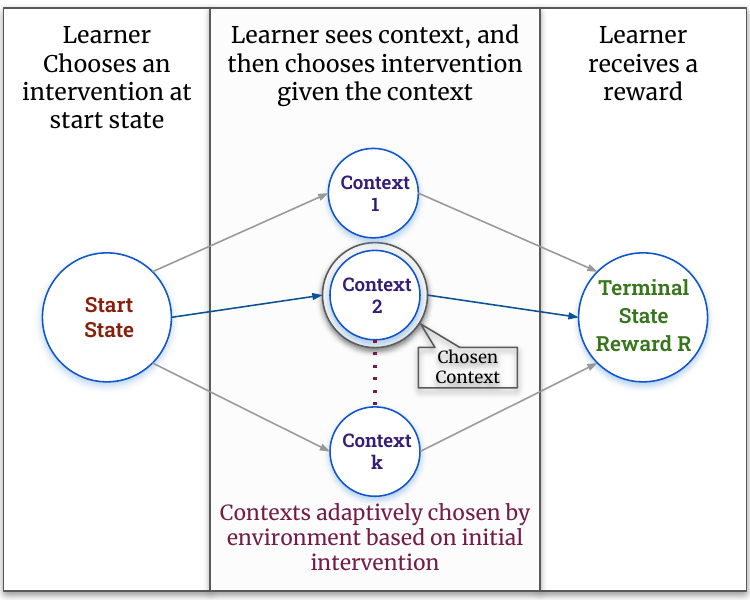}
    }
    \hfill
    \subfigure[Illustrative Figure for Causal Graph at start state and at some intermediate context\texorpdfstring{ $i\in[k]$.}{.}]{
    \label{figure: causal graphs}
    \includegraphics[width=0.47\textwidth]{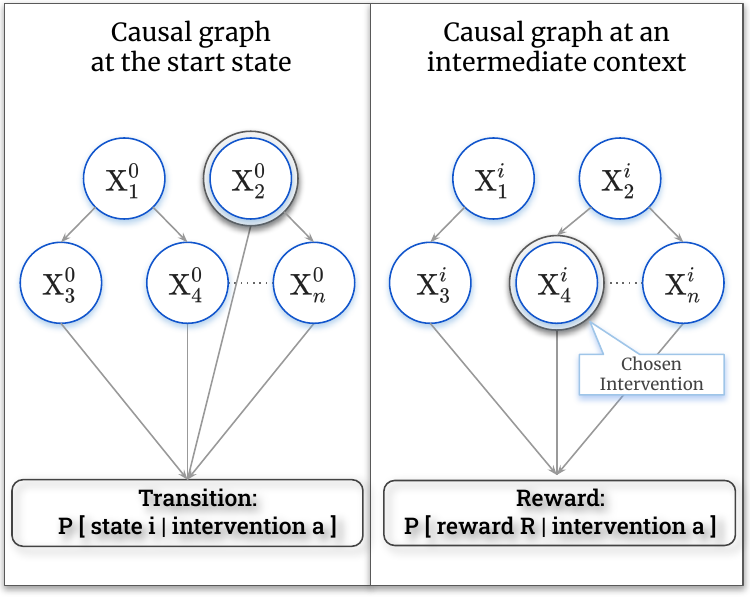}
    }
    \caption{The transition to a particular context (chosen context in the figure on the left) is decided by the environment, whereas the interventions at the start state and an intermediate context  (chosen interventions in the figure on the right) are chosen by the learner.}
\end{figure*}

	We model the causal contextual bandit with adaptive context as a contextual bandit problem with a causal graph corresponding to each context. The actions at each context are given by interventions on the causal graph. Additionally, we have a causal graph at the start state, and the context is stochastically dependent on the intervention on the causal graph at the start state. For ease of notation, we will call the start state of the learner as context $0$. 
    The agent starts at context $0$, chooses an intervention, then transitions to one of $k$ contexts $[k] =\{1,\dots,k\}$, chooses another intervention, and then receives a reward; see \hyperref[figure: two-layer MDP]{Figure \ref{figure: two-layer MDP}}.

    \textbf{Assumptions on the Causal Graph:} \label{assumptions: model assumptions} Formally, let $\C$ be the set of contexts $\{0,1,\dots,k\}$. Then, at each context, there is a Causal Bayesian Network (CBN) represented by a causal graph; see \hyperref[figure: causal graphs]{Figure \ref{figure: causal graphs}}.
    In particular, at each context $i \in \C$, the causal graph is composed of $n$   variables $\{X^i_1,\dots,X^i_n\}$. Each $X^i_j$ takes values from $\{0,1\}$, with an associated conditional probability (of being equal to 0 or 1), given the other variables in the causal graph.     
    We make the following mild assumptions on the causal graph at each context. \vspace{-0.05in}
    \begin{enumerate}
        \item The distribution of any node $X_i$ conditioned on it's parents in the causal graph is a Bernoulli random variable with a fixed parameter.
        \item The causal graph at each context is semi-Markovian. This is equivalent to making the following assumptions on the graph. No hidden variable in the graph has a parent. Further, every hidden variable has at most two children, both observable.
        \item We transform the causal graph for each context as follows: For every hidden variable with two children, we introduce bidirected edges between them. If no path of bidirected edges exists between an intervenable node and its child, the graph is identifiable -- a necessary and sufficient condition for estimating the graph's associated distribution.\citep{tian2002testable}.
    \end{enumerate}

\begin{table*}[!t]
\renewcommand{\arraystretch}{1.25}
    \small
        \caption{Summary of notations for our paper} \label{table: causal notations table}
        \begin{center} 
		\begin{tabular}{|x{2.4cm}|x{10cm}|}
		    \hline
			\hline
			\textbf{Notation} & \textbf{Explanation}\\
			\hline
			Context $0$ & Start state\\
			Context $[k]$ & Intermediate contexts $\{1,\dots,k\}$\\
			\hline
			$X^i_j$ & Causal Variables: $X^i_j\in \{0,1\}\enspace \text{ for all } i \in [k],\enspace j\in[n]$\\
			\hline
			$do(\cdot)$&An atomic intervention of the form $do()$, $do(X_j^i=0)$ or $do(X_j^i=1)$\\
			\hline
			$\I_i$ & Set of atomic interventions at context $i$ \\
			$N$ & $N:= |\I_i| = 2n+1\enspace \text{ for all } i \in [k]$\\
			\hline
			$R_i$ & Reward on transition from context $i$\\
			\hline
			$m_i$& Causal observational threshold at context $i \in \{0,\dots,k\}$\\
			$M$ & diagonal matrix of $m_i$ values\\
            \hline
            $P\in\R^{N\times k}$& Transition probabilities matrix: $\left[ P_{(a,i)} = \prob\{ i \enspace \mid \enspace a \} \right]_{a \in \I_0, i\in[k] }$\\
            \hline
            
            $p_+$ &  Transition threshold $p_+ = \min\{P_{(a,i)} \mid  P_{(a,i)} > 0\}$\\
            \hline
            $\pi:\S \to \I$ & Policy, a map from contexts to interventions.\newline i.e. $\pi(i)\in\I_i$ for $i\in \{0\}\cup [k]$\\
            \hline
            $\E\left[R_i  \mid  \pi(i) \right]$ & Expectation of the reward at context $i$ given intervention $\pi(i)$ \\ 
			\hline
			\hline
		\end{tabular}
	\end{center}
 \renewcommand{\arraystretch}{1}
	\end{table*}

    \textbf{Interventions:} Furthermore, we are allowed atomic interventions, i.e., we can select \emph{at most} one variable and set it to either $0$ or $1$. We will use $\I_i$ to denote the set of atomic interventions available at context $i\in \{0, \ldots,k\}$; in particular, $\I_i = \left\{do() \right\} \cup \left\{do(X^i_j=0), do(X^i_j=1)\right\}$ for  $\ j \in [n] $. We note that $do()$ is an empty intervention that allows all the variables to take values from their underlying conditional distributions. Also, $do(X^i_j=0)$ and $do(X^i_j=1)$ set the value of variable $X^i_j$ to $0$ and $1$, respectively, while leaving all the other variables to independently draw values from their respective distributions. Note that for all $ i \in [k]$, we have $\lvert \I_i \rvert = 2n+1$. Write $N:= 2n+1$.
	
	\textbf{Reward:} The environment provides the learner with a $\{0,1\}$ reward upon choosing an intervention at context $i\in[k]$, which we denote as $R_i$.
    Note that $R_i$ is a stochastic function of variables $X^i_1,\dots, X^i_n$. In particular, for all $j \in [n]$ and each realization $X^i_j = x_j \in \{0,1\}$, the reward $R_i$ is distributed as $\prob \{R_i=1 \mid X^i_1 = x_1, \ldots$ $,X^i_n = x_n \}$. 
    
    Given such conditional probabilities, we will write $\E[R_i \mid a]$ to denote the expected value of reward $R_i$ when intervention $a \in \I_i$ is performed at context $i \in [k]$. Here the expectation is over the parents of the variable $R_i$ in the causal graph, with the intervened variable set at the required value. Note that these parents (of $R_i$) may in turn have conditional distributions given their parents. The leaf nodes of the causal graph are considered to have unconditional Bernoulli distributions.
    For instance, $\E[R_i \mid do(X^i_j=1)]$ is the expected reward when variable $X^i_j$ is set to $1$, and all the other variables independently draw values from their respective (conditional) distributions. Indeed, the goal of this work is to develop an algorithm that maximizes the expected reward at context $0$.
	

    \textbf{Causal Observational Threshold:} We denote by $m_i$, the causal observational threshold\footnote{\cite{maiti2022causal} extend the causal observational threshold from \cite{Lattimore} to the general setting of causal graphs with unobserved confounders} from \cite{maiti2022causal} at context $i$. This is computed as follows. Let $\widehat{q}_j^i = \min_{\text{Parents}(X_j^i),x\in\{0,1\}} \prob\{X_j^i = x \mid \text{Parents}(X_j^i)\}$.
    Further, let $S^i_\tau = \{\widehat{q}_j^i: (\widehat{q}_j^i)^c < 1/\tau\}$ be sets parameterized by $\tau$ for every $\tau \in [2,2n]$, where $c$ indicates the c-component size.
    Then $m_i = \min \{ \tau \text{ such that } |S^i_\tau | \leq \tau\}$.
    The existence of such a threshold at each context is guaranteed by the \hyperref[assumptions: model assumptions]{assumptions} we made on the CBNs. In addition, let $\smash{ M\in\N^{k\times k} }$ denote the diagonal matrix of $m_1, \ldots, m_k$.

	\textbf{Transitions at Context 0:} 
    At context $0$, the transition to the intermediate contexts $[k]$ stochastically depends on the random variables $\{X^0_1,\dots,X^0_n\}$. Here, $\prob \{i \enspace\mid\enspace a\}$ denotes the probability of transitioning into context $i \in [k]$ with atomic intervention $a \in \I_0$; recall that $\I_0$ includes the do-nothing intervention. We will collectively denote these transition probabilities as matrix $P:= \left[ P_{(a,i)} = \prob \{ i \enspace\mid\enspace a \} \right]_{a \in \I_0 , i\in[k]}$. Furthermore, write the transition threshold $p_+$ to denote the minimum non-zero value in $P$. Note that matrix $P \in \R^{\lvert \I_0 \rvert \times k}$ is fixed, but unknown. 
	
	
	\textbf{Policy:} A map $\pi:\{0,\dots,k\} \to \I$, between contexts and interventions (performed by the algorithm), will be referred to as a policy. Specifically, $\pi(i) \in \I_i$ is the intervention at context $i\in\{0,1, \ldots, k\}$. Note that, for any policy ${\pi}$, the expected reward, which we denote as $\mu(\pi)$, is equal to $\sum_{i=1}^k \E \left[R_i \ \mid \ {\pi}(i) \right] \cdot \prob \{ i \ \mid \ \pi(0) \}$. 
	Maximizing expected reward, at each intermediate context $i \in [k]$, we obtain the overall optimal policy $\pi^*$ as follows. For $i \in [k]$: 
	
	\begin{align}
	    \pi^*(i) &= \argmax_{a \in \I_i} \ \E \left[ R_i \mid a\right]\vspace{-0.05in}\\
	    \vspace{-0.05in}
	    \pi^*(0) &= \argmax_{b \in\I_0} ( \sum_{i=1}^k \  \E \left[R_i\mid \pi^*(i) \right] \cdot \prob \{ i\mid b \} )
	\end{align}

	\vspace{-0.05in}
	Our goal then is to find a policy $\pi$ with (expected) reward as close to that of $\pi^*$ as possible.

    \textbf{Simple Regret:} 
	Conforming to the standard \emph{simple-regret} framework, the algorithm is given a time budget $T$, i.e., the learner can go through the following process $T$ times --- (a) start at context $0$. (b) Choose an intervention $a \in \I_0$. (c) Transition to context $i \in [k]$. (d) Choose an intervention $a \in \I_i$. (e) Receive reward $R_i$. At the end of these $T$ steps, the goal of the learner is to compute a policy. Let the policy returned by the learner be $\widehat{\pi}$. Then the simple regret is defined as the expected value: $\E[\mu(\pi^*) - \mu(\widehat{\pi}]$. Our algorithm seeks to minimize such a simple regret.

%% file: Includes/algorithm.tex
\vspace{-0.05in}
\section{Main Algorithm and its Analysis}
\vspace{-0.05in}
	We now provide the details relating to our main Algorithm, viz. $\CE$. 

    \vspace{0.05in}
    \begin{figure*}[!thbp]
    \begin{center}
    \begin{minipage}[t]{1\textwidth}
    \begin{algorithm}[H]
    \small
		\caption{$\CE$: Convex Exploration Algorithm}
		\label{alg:best policy generator}
		\begin{algorithmic}[1] 
			\State \textbf{Input:} Total rounds $T$
			\State \hyperref[alg:estimateTransitionProbabilities]{Estimate the transition probabilities} $\widehat P$ from the start state to the intermediate contexts for time $T/3$, by  performing interventions at context 0 in a round robin manner.
            \State \hyperref[alg:estimateCausalParameters]{Estimate the causal observational threshold matrix} $\hat M$ for time $T/3$, by  performing interventions at context 0 as per frequency vector $\tilde{f}$ where $\tilde{f} \gets \argmax\limits_{\text{fq.~vector } f } \enspace \min\limits_{\text{contexts [k]}}  \widehat{P}^\tr  f$.
            \State \hyperref[alg:estimateRewards]{Estimate the reward matrix} $\widehat{\mathcal{R}}$ for time $T/3$, by  performing interventions \footnote{Computation of $\widehat{f}^*$ is efficient as we show that the problem is \hyperref[lem:optimization problem is convex]{Convex}.} at context 0 as per frequency vector $\widehat{f}^*$  where $\widehat{f}^* \gets \argmin\limits_{\text{fq.~vector } f } \, \max\limits_{\text{interventions } \mathcal{I}_0} \widehat{P}\hat{M}^{1/2}\left(\widehat{P}^\tr f \right)^{\circ-\frac{1}{2}}$. \label{step:fstar}
   
			\State \textcolor{AlgHighlight}{\textbf{Estimate the optimal action at each intermediate context}} $\widehat{\pi}(i) \enspace \forall i \in [k]$ based on $\widehat{\mathcal{R}}$. Let the estimate of optimal reward be $\widehat{\mathcal{R}}(\widehat{\pi}(i))$.
			\State \textcolor{AlgHighlight}{\textbf{Estimate the optimal action at the start context}} $\widehat{\pi}(0)$, based on the transition probabilities $\widehat{P}$ and the optimal reward estimates $\widehat{\mathcal{R}}(\widehat{\pi}(i))$.
			\State \textbf{return} $\widehat{\pi} = \{\widehat{\pi}(0),\widehat{\pi}(1),\dots,\widehat{\pi}(k)\}$
			\footnotetext[2]{We show detailed Algorithms for estimation of transition probabilities $P$ (line 2), estimation of causal observational threshold $M$ (line 3), and estimation of rewards $\mathcal{R}$ (line 4) in Appendix \ref{appendixsection:Algorithms in detail}}.
		\end{algorithmic}
	\end{algorithm}
	\end{minipage}
	\end{center}
    \end{figure*}

    The algorithm can be described by five main steps. In the first step, we estimate the transitions to intermediate contexts. In the second step, we estimate the causal observational thresholds at these contexts. In the third step, we estimate the rewards upon doing interventions at these contexts. With good reward estimates and transition probability estimates, the computation of a good policy at the intermediate contexts (step 4) and at the start state (step 5) is straightforward. 
    This Algorithm relies on three subroutines which are detailed in Section \ref{appendixsection:Algorithms in detail} of the Appendix. The key aspect of this algorithm is in designing the exploration of interventions (at the start state and at the intermediate contexts) to be regret-optimal -- i.e. trading off exploration time between different interventions such that the policy eventually obtained has near-optimal reward.

	Our algorithm (\hyperref[alg:best policy generator]{\CE}) uses subroutines to estimate the transition probabilities, the causal parameters,  and the rewards. From these, it outputs the best available interventions as its policy $\widehat{\pi}$. Given time budget $T$, the algorithm uses the first $T/3$ rounds to estimate the transition probabilities (i.e., the matrix $P$) in \hyperref[alg:estimateTransitionProbabilities]{Algorithm \ref{alg:estimateTransitionProbabilities}}. The subsequent $T/3$ rounds are utilized in \hyperref[alg:estimateCausalParameters]{Algorithm \ref{alg:estimateCausalParameters}} to estimate  causal parameters $m_i$s. Finally, the remaining budget is used in \hyperref[alg:estimateRewards]{Algorithm \ref{alg:estimateRewards}} to estimate the intervention-dependent reward $R_i$s, for all intermediate contexts $i\in[k]$.

    To judiciously explore the interventions at context $0$, \hyperref[alg:best policy generator]{\CE} computes frequency vectors $f\in\R^{\lvert  \I_0 \rvert}$. In such vectors, the $a$th component $f_a \geq 0$ denotes the fraction of time that each intervention $a \in \I_0$ is performed by the algorithm, i.e., given time budget $T'$, the intervention $a$ will be performed $f_a T'$ times. Note that, by definition, $\sum_a f_a = 1$ and the frequency vectors are computed by solving convex programs over the estimates. The algorithm and its subroutines throughout consider empirical estimates, i.e., find the estimates by direct counting. Here, let $\widehat{P}$ denote the computed estimate of the matrix $P$ and $\smash{\hat{M}}$ be the estimate of the diagonal matrix $M$.  We obtain a regret upper bound via an optimal frequency vector $\widehat{f}^*$ (see Step \ref{step:fstar} in \hyperref[alg:best policy generator]{\CE}).

	Recall that for any vector $x$ (with non-negative components), the Hadamard exponentiation ${\circ-0.5}$ leads to the vector $y = x^{\circ-0.5}$ wherein $y_i = 1/\sqrt{x_i}$ for each component $i$.  
	We next define a key parameter $\lambda$ that specifies the regret bound in Theorem \ref{theorem:main} (below). 
	At a high-level, parameter $\lambda$ captures the ``exploration efficacy'' in the MDP, that takes into account the transition probabilities $P$ and the exploration requirements $M$ at the intermediate layer. Identification of this parameter is a relevant technical contribution of our work; see Section \ref{appendix section: proof of theorem 1} for a detailed derivation of $\lambda$.
	
	\begin{align}
		\lambda := \min_{\text{fq.~vector} f} \ \left\lVert PM^{0.5}  \left(P^\tr  f \right)^{\circ-0.5} \right\rVert_\infty^2 \label{eqn:lambda}
	\end{align}
	
	Furthermore, we will write $f^*$ to denote the optimal frequency vector in equation (\ref{eqn:lambda}). Hence, with vector $\nu := P{M}^{0.5}  (P^\tr  f^* )^{\circ-0.5}$, we have $\lambda = \max_a  \nu_a^2$. 
	Note that  Step \ref{step:fstar} in \hyperref[alg:best policy generator]{\CE} addresses an analogous optimization problem, albeit with the estimates $\smash {\widehat{P}}$ and $ \hat{M}$. Also, we show in \hyperref[lem:optimization problem is convex]{Lemma \ref{lem:optimization problem is convex}} (see Section \ref{section: nature of optimization problems} in the supplementary material) that this optimization problem is convex and, hence, Step \ref{step:fstar} admits an efficient implementation.
	
	To understand the behaviour of $\lambda$,
	we first note that whenever the $m_i$ values at the contexts $i\in[k]$ are low, the $\lambda$ value is low. Specifically, the $m_i$ values can go as low as $2$ (when the $q^i_j$s are all $\frac{1}{2}$), removing the dependence of $\lambda$ on $n$. The upper-bound on $\lambda$ is $nk$. We see this by first upper-bounding each $m_i$ by $n$. Then, note that whenever $\max_{a\in\A}P\{i|a\} \geq 1/k$, then $\exists f$ such that $P^\top f = u$ where $u=\{\frac{1}{k},\dots,\frac{1}{k}\}$. Now we can compute that $||P\cdot u^{\circ -0.5}||^2_{\infty} = k$, and thereby $\lambda < nk$; See footnote\footnotemark[2].
	
	\footnotetext[2]{$\lambda$ is upperbounded by kn, but is typically significantly smaller (as m may be much smaller than n).}

	The following theorem that upper bounds the regret of \hyperref[alg:best policy generator]{\CE} is the main result of the current work. The result requires the algorithm's time budget to be at least 
	$\label{eqn: T0}	    T_0 := \widetilde{O}\left(N\max(m_i)/p_+^3\right)$

	\begin{thm} \label{theorem:main}
		Given number of rounds $T \geq T_0$ and $\lambda$ as in equation (\ref{eqn:lambda}), \hyperref[alg:best policy generator]{\CE} achieves regret 
		$$\Regret_T \in \O\left(\sqrt{\max\left\{\frac{\lambda}{T},\frac{m_0}{Tp_+}\right\}\log\left(NT\right)}\right)$$
		
	\end{thm}
	Observe that $m_0/Tp_+$ is independent of the number of contexts and interventions. Therefore $\lambda$ dominates when number of interventions at an  intermediate context is large.
	

%% file: Includes/lower-bound.tex
\section{Analysis of the Lower Bound}
    \label{section: lower bounds}
    
    Since $\CE$ solves an optimization problem, it is a priori unclear that a better algorithm may not provide a regret guarantee better than Theorem \ref{theorem:main}. In this section, we show that  for a large class of instances, it is indeed the case that the regret guarantee we provide is optimal. 
    We provide a lower bound on regret for a family of instances. For any number of contexts $k$, we show that there exist transition matrices $P$ and reward distributions ($\E[R_i \mid a]$) such that regret achieved by \CE $\ $(Theorem \ref{theorem:main}) is tight, up to log factors. 
    
    \begin{thm}
    \label{theorem: Lower Bound for our algorithm}
    For any $q^i_j$ corresponding to causal variables at contexts $i \in [k]$,
    there exists a transition matrix $P$, and probabilities $q^0_j$ corresponding to causal variables $\{X^0_j\}_{j\in[n]}$, and reward distributions, such that the simple regret achieved by \textit{any} algorithm is 
    $$\Regret_T\in \Omega\left(\sqrt{\frac{\lambda}{T}}\right)$$
    
    \end{thm}
\vspace{-0.05in}
    We provide the details of the proof of Theorem \ref{theorem: Lower Bound for our algorithm} in Section \ref{appendix section: analysis of lower bounds} in the supplementary material.

%% file: Includes/experiments.tex
    \vspace{-0.12in}
	\section{Experiments}
	\label{section: experiments}
	
    \vspace{-0.1in}

	We first list a few baseline algorithms that we compare \hyperref[alg:best policy generator]{\CE} with. This is followed by a complete description of our experimental setup. Finally, we present and discuss our main results. 
	
	\label{section: UE Description}
	\textbf{Uniform Exploration:} This algorithm uniformly explores the interventions in the instance. It first performs all the atomic interventions $a \in \I_0$ at the start state $0$ in a round robin manner. On transitioning to any context $i\in [k]$, it performs atomic interventions $b\in\I_i$ in a round robin manner. $\UE$ achieves a regret upperbounded by $\tilde{\O}(\sqrt{nk/T})$, which is also the optimal lower bound for non-causal algorithms. Hence it serves as a good comparison as it achieves an optimal non-causal simple regret. 
    We plot the comparison with this non-causal regret optimal exploration in Figure \ref{figure: experimental results}. We plot the regret with respect to (A) the number of rounds of exploration and (B) with the $\lambda$ values of our instance. Notice that at extremely high $\lambda$ values $\CE$ does not perform well, as such an instance does not particularly benefit from the causal structure. Even so, with further tuning of constants in our Algorithm, we should achieve a performance similar to $\UE$.
	
\begin{figure*}[!bht]
\centering
\includegraphics[width=0.98\linewidth]{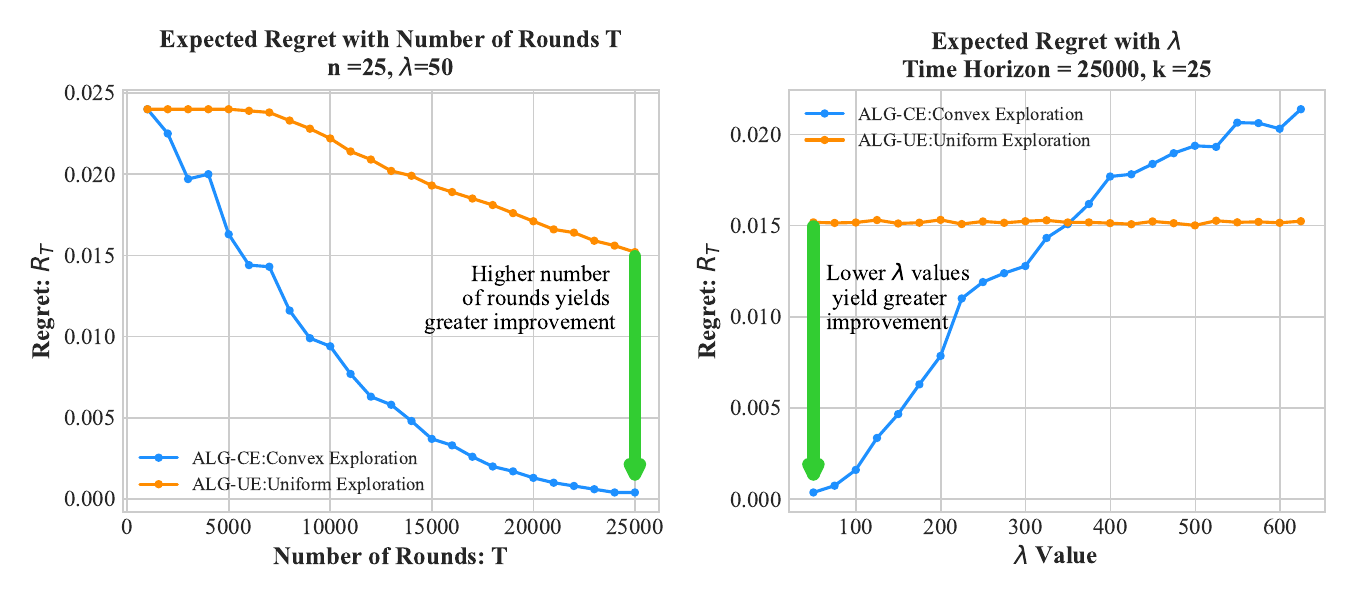} 
\caption{We plot the Simple Regret under $\CE$ and $\UE$. The figure on the left (\ref{figure: experimental results}a) plots      expected simple regret vs time, for the setup $n=25$,  $k=25$, $\lambda=50$, $\varepsilon = 0.3$ and $m=2$ for all contexts. The figure on the right (\ref{figure: experimental results}b) plots expected simple regret with $\lambda$. It was performed with the parameters: $T=25000$, $k=25$, $m_0 = 2$ and $\varepsilon = 0.3$.}
\label{figure: experimental results}
\end{figure*}

 \textbf{Other Baselines:} We now consider several other baselines for comparison, that have been used in literature. Primary amongst these are: (1) UCB at the start state, as well as the intermediate contexts (2) Thompson sampling at the start state, as well as the intermediate contexts (3) Round-robin at the start state, and UCB at the intermediate contexts (4) Round-robin at the start state, and Thomson sampling at the intermediate contexts and (5) $\UE$ which is round-robin at both the start state and at the intermediate contexts.

	\textbf{Setup:} We consider $k=25$ intermediate contexts and a causal graphs with $n=25$ variables ($2n+1 = 51$ interventions) at each context. The rewards are distributed Bernoulli($0.5+\varepsilon$) for intervention $X^1_1 = 1$ and Bernoulli($0.5$) otherwise where $\varepsilon=0.3$ in the experiments. We set $m_i = m \enspace \forall i\in[k]$. As in experiments in prior work, we set $q^i_j=0$ for $j\leq m_i$ and $0.5$ otherwise.
	Let $k=n$ here. At state $0$,
	on taking action $a = do()$, we transition uniformly to one of the intermediate contexts. On taking action $do(X^0_i=1)$, we transition with probability $2/k$ to context $i$ and probability $1/k - 1/(k(k-1))$ to any of the other $k-1$ contexts. 
	
	
	We perform two experiments in this setting. In the first one, we run \hyperref[alg:best policy generator]{\CE} and \hyperref[section: UE Description]{\UE} for  time horizon $T\in\{1000,\dots, 25000\}$. In the second experiment, we run \hyperref[alg:best policy generator]{\CE} and \hyperref[section: UE Description]{\UE} for a fixed time horizon $T=25000$ with $\lambda$ varying in the set $\{50, 75, \dots, 625\}$. To vary $\lambda$, we vary $m_i$ for the intermediate contexts in the set $\{2,3,\ldots,25\}$. We average the regret over $10000$ runs for each setting. We use CVXPY (\cite{CVXOPT}) to solve the convex program at \hyperref[step:fstar]{Step \ref{step:fstar}} in \hyperref[alg:best policy generator]{\CE}.
 We release our code in entirety in our anonymized \hyperlink{https://github.com/adaptiveContextualCausalBandits/aCCB}{GitHub project repository}, for the community to use and improve.

\begin{figure*}
\centering
\includegraphics[width=1.05\linewidth]{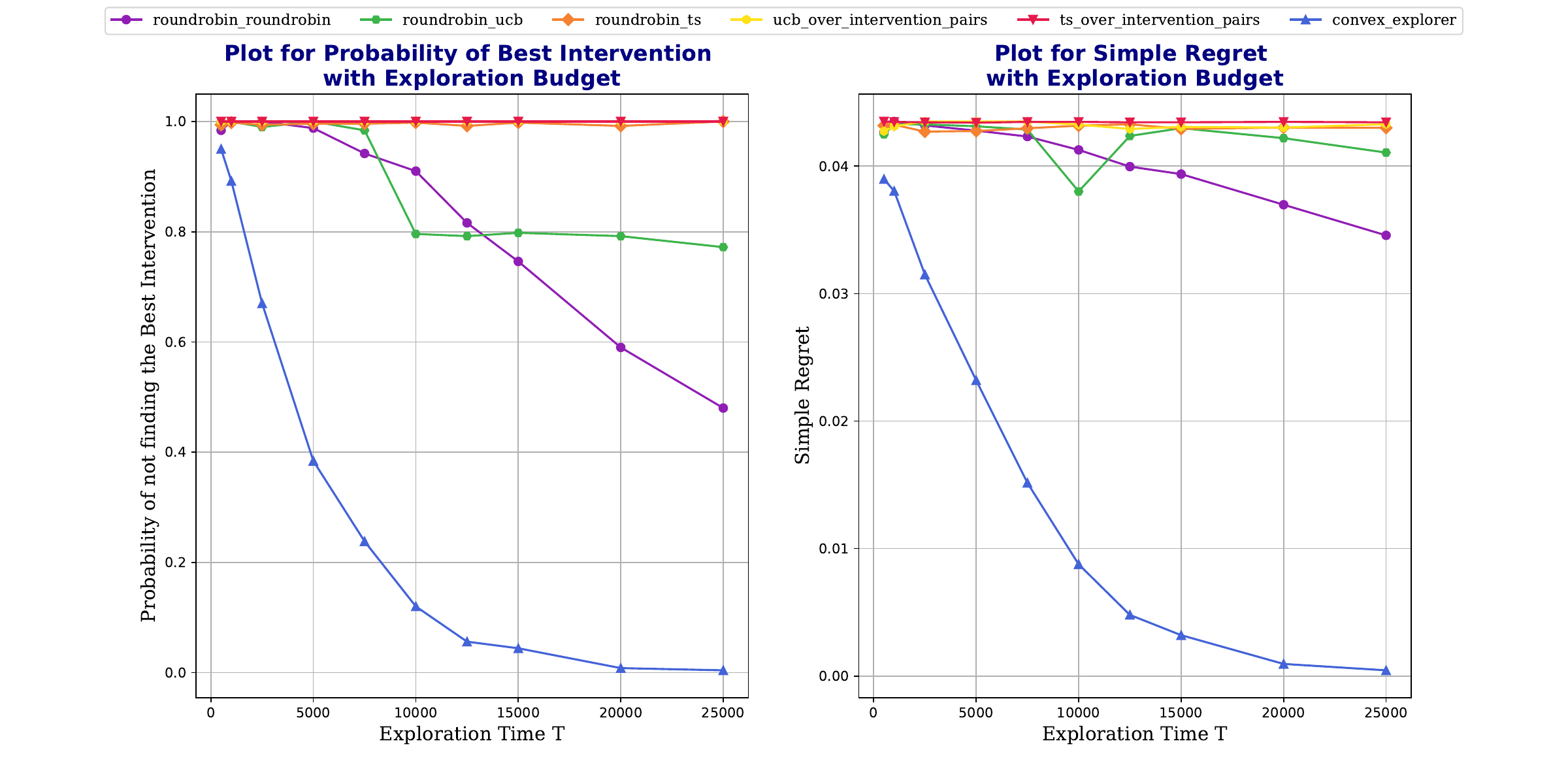} 
\caption{We plot various baselines for two metrics of interest (1) Probability of the algorithm finding the best interventions and (2) Simple regret. These plots illustrate how these metrics vary with the exploration budget.}
\label{figure: experimental results2}
\end{figure*}

\begin{figure*}[!bht]
\centering
\includegraphics[width=1.05\linewidth]{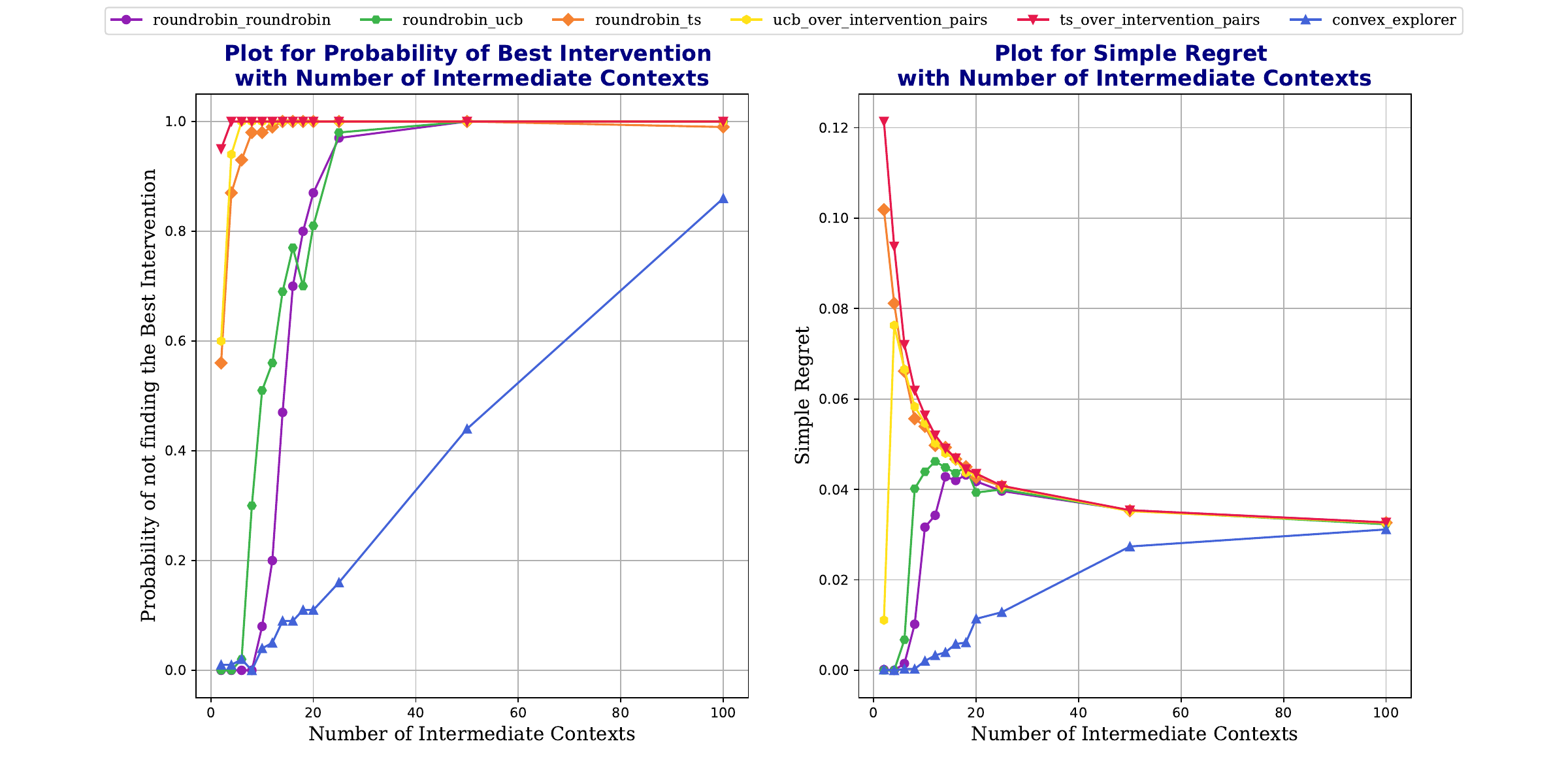} 
\caption{We plot the variation of probability of finding the best intervention and simple regret with the number of contexts. Notice the outperformance of $\CE$ vs. the other baselines.}
\label{figure: experimental results3}
\end{figure*}

	\textbf{Results of comparison with $\UE$:} In \hyperref[figure: experimental results]{Figure \ref{figure: experimental results}a}, we compare the expected simple regret of \hyperref[alg:best policy generator]{\CE} vs. \hyperref[section: UE Description]{\UE}. Our plots indicate that \hyperref[alg:best policy generator]{\CE} outperforms \hyperref[section: UE Description]{\UE} and its regret falls rapidly as $T$ increases.  In \hyperref[figure: experimental results]{Figure \ref{figure: experimental results}b}, we plot the expected simple regret against $\lambda$ for \hyperref[alg:best policy generator]{\CE} and \hyperref[section: UE Description]{\UE} that was obtained in Experiment $2$, and empirically validate their relationship that was proved in Theorem \ref{theorem:main}. 
	

\begin{figure*}
\centering
\includegraphics[width=1.05\linewidth]{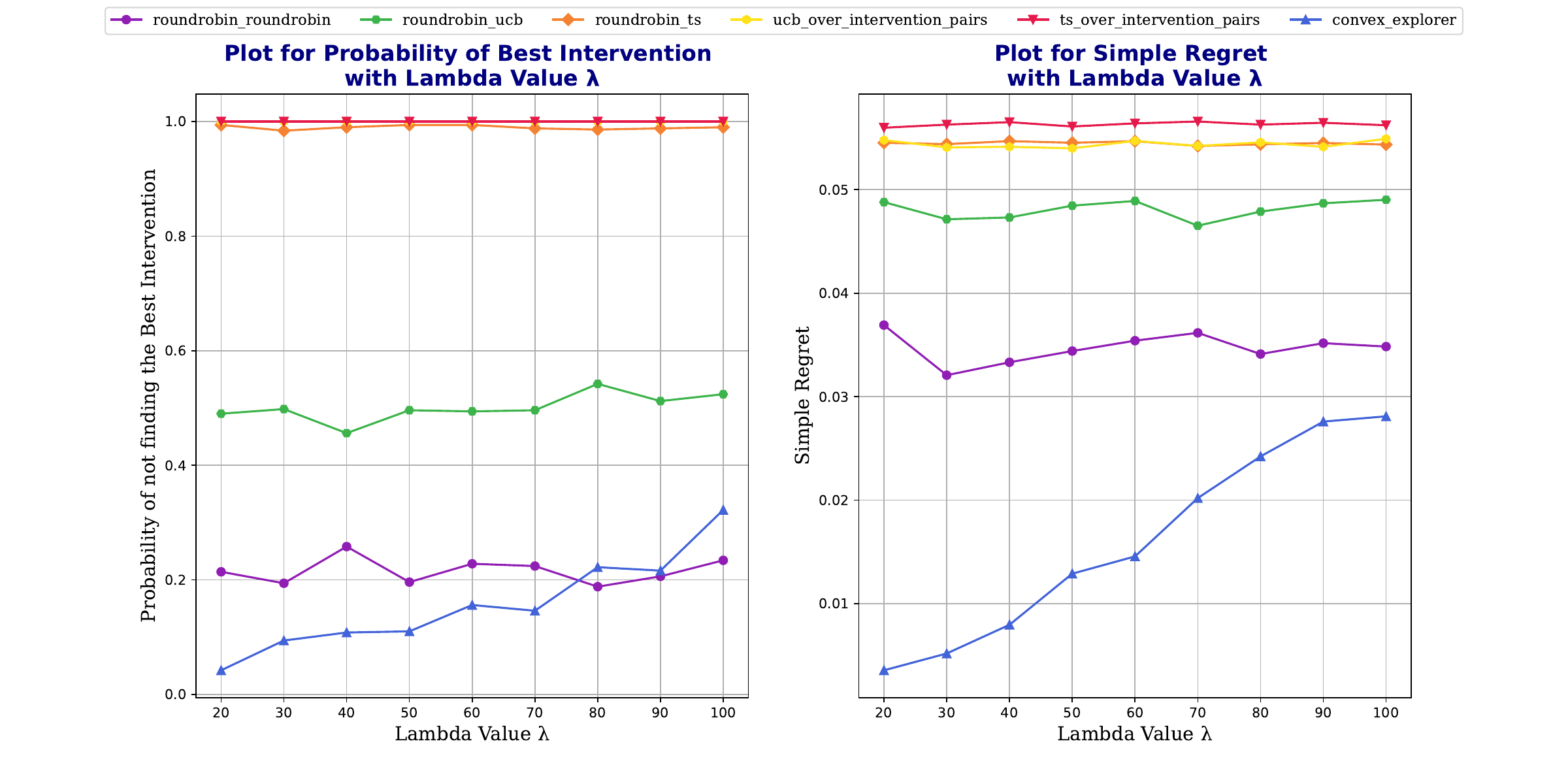} 
\caption{We plot the variation of probability of finding the best intervention and simple regret with $\lambda$ value. Notice that $\CE$ is the only algorithm that is causal-aware and hence varying with $\lambda$.}
\label{figure: experimental results4}
\end{figure*}

 \textbf{Results of comparison withother baselines:} We find that $\CE$ significantly outperforms baselines other than $\UE$. Specifically Thompson samplling and UCB are not well tuned to the exploration problem, and hence perform poorly in both the metrics of (1) simple regret as well as (2) probability of finding the best intervention. A mixture of round-robin at the start state with these alternatives at the intermediate context also perform poorly with respect to $\CE$ for this particular exploration problem. In Figure \ref{figure: experimental results2} we plot the metrics with exploration budget. In Figure \ref{figure: experimental results3} we plot the metrics of interest with the number of contexts at the intermediate stage. Finally, in Figure \ref{figure: experimental results4}, we plot the simple regret as well as probability of finding the best intervention with our parameter $\lambda$, while keeping the number of intermediate contexts the same. The results of these experiments and full details can be found \hyperlink{https://github.com/adaptiveContextualCausalBandits/aCCB}{here}. 

%% file: Includes/conclusion.tex
\vspace{-0.1in}
\section{Conclusions}
\label{sec: conclusion and future work}

	\vspace{-0.1in}
	We studied extensions of the causal contextual bandits framework to include adaptive context choice. This is an important problem in practice and the solutions therein have immediate practical applications.
    The setting of stochastic transition to a context accounted for non-trivial extensions from \cite{subramanian2022causal} who studied targeted interventions.
	We developed a Convex Exploration algorithm for minimizing simple regret under this setting. Furthermore, while \cite{maiti2022causal} studied the simple causal bandit setting with unobserved confounders, our work addresses causal contextual bandits with adaptive contexts, under the same constraint of allowing unobserved confounders (assuming identifiability).
	We identified an instance dependent parameter $\lambda$, and proved that the regret of this algorithm is $\tilde{\O}(\sqrt{\frac{1}{T}\max\{\lambda,\frac{m_0}{p_+}\}})$. The current work also established that, for certain families of instances, this upper bound is essentially tight. Finally, we showed through experiments that our algorithm performs better than uniform  exploration in a range of settings.
	We believe our method of converting the exploration in the causal contextual bandit setting is novel, and may have implications outside the causal setting as well.

    Possible generalizations of this work include extensions to non-binary reward settings. Another natural extension would be to derive bounds for L-layered MDPs, extending from the adaptive contextual bandit setting we consider. It would be interesting to see whether that problem reduces to convex exploration as well. Finally, extending convex exploration methods from this paper to other more general simple regret problems may also be a promising avenue for future research.
    


%% file: Includes/acknowledgements.tex
\section{Acknowledgements}
Siddharth Barman gratefully acknowledges the support of the Walmart Center for Tech Excellence (CSR WMGT-23-0001) and a SERB Core research grant (CRG/2021/006165).

%% file: appendix/appendixRelatedWork.tex
\section{Related Work}
\label{appendix sec: related work appendix}
In our work, we draw from prior literature from causality as well as from multi-armed bandits. We will briefly cover these two in the following section.

\subsection{Multi-armed bandits:}
The stochastic Multi-Armed Bandit (MAB) setup is a standard model for studying the exploration-exploitation trade-off in sequential decision making problems \citep{kuleshov2014algorithms,bubeck2012regret}. Such trade-offs arise in several modern
applications, such as ad placement, website optimization, recommendation systems, and packet routing \citep{bouneffouf2020survey} and are thus a central part of the theory relating to online learning \citep{slivkins2019introduction,lattimore2020bandit}.

Traditional performance measures for MAB algorithms have focused on cumulative regret \citep{auer2002FTAMAB,agrawal2012analysis,auer2010ucb}, as well as best-arm identification under the fixed confidence \citep{even2006action} and fixed budget \citep{audibert2010best} settings.
In some settings however, one may be interested in optimizing the exploration phase. Another variant of regret that has been considered is the mini-max regret \citep{azar2017minimax} which focuses on the worst case over all possible environments.
However, as a metric for pure exploration in MABs, simple regret has been proposed as a natural performance criterion \citep{bubeck2009pure}. In this setting, we allow for some period of exploration, after which the learner has to choose an arm. The simple regret is then evaluated as the difference between the average reward of the best arm and the average reward of the learner's recommendation. We focus on simple regret in this work.

Each of these performance metrics come with their own lower bounds \citep{orabona2012beyond,osband2016lower,bubeck2012regret}, which are naturally the benchmarks for any algorithms proposed. The lower bound on simple regret is known to be $\O(\sqrt{n/T})$ for a stochastic multi-armed bandit problem with $n$ arms. This bound is obtained from the lower bound for pure exploration provided by \citet{mannor2004sample}. 

Note that, a naive approach to the causal bandit problem which simply treats an intervention on each of exponentially many combinations of the nodes as an arm, may thus incur an exponential regret.
We now review some of the literature from Causality, which helps in addressing the causal aspects of the problem.

\subsection{Causality:}
There are three broad threads in causality related to our work. These are causal graph learning, causal testing and causal bandits. We address relevant works in these areas below.

\textbf{Learning Causal Graphs:} \citet{tian2002testable} laid the grounds for analysing functional functional constraints among the distributions of observed variables in a causal Bayesian networks. Similarly, \citet{kang2006inequality} derive such functional constraints over interventional distributions. These two seminal works lead to a great interest in the problem of learning causal graphs. 

There have been several studies that provide algorithms to recover the causal graphs from the conditional independence relations in observational data \citep{pearl1995theory, spirtes2000causation, ali2005towards, zhang2008completeness}. Subsequent work considered the setting when both observational and interventional data are available \citep{eberhardt2005number, hauser2014two}. \cite{kocaoglu2017cost} extend the causal graph learning problem to a budgeted setting. \citet{shanmugam2015learning} uses interventions on sets of small size to learn the causal structure. 
\citet{kocaoglu2017experimental} provide an efficient randomized algorithm to learn a causal graph with confounding variables.

\textbf{Testing over Bayesian networks:} 
Given sample access to an unknown Bayesian Network \citep{canonne2017testing}, or Ising model \citep{daskalakis2019testing}, one may wish to decide whether an unknown model is equal to a known fixed model, and analyse the sample complexity of this hypothesis test. \citet{acharya2018learning} address this question by introducing the concept of covering interventions. These covering interventions allow us to understand the behaviour of multiple interventions (that are covered) simultaneously. We utilize the concept of covering interventions from \citet{acharya2018learning} towards our question of finding the optimal intervention in a causal bandit. The area of reinforcement learning over causal bandits has also been studied in \citet{pmlr-v119-zhang20a}.

Apart from these areas in causality, our primary problem of causal bandits have been addressed by \citet{Lattimore,maiti2022causal,sen2017identifying,lu2020regret,Gaurav2020,sen2017contextual,lu2021causal,lu2022efficient,varici2022causal,xiong2023combinatorial}. We detail these in the main Related Works \hyperref[sec: related work]{Section \ref{sec: related work}}.

%% file: appendix/appendixAlgorithmsInDetail.tex
\section{Algorithms in Detail}
\label{appendixsection:Algorithms in detail}
In this section, we outline the three algorithms that are used as helpers in \hyperref[alg:best policy generator]{\CE}. The first that we outline now, Algorithm \ref{alg:estimateTransitionProbabilities}, would be used to estimate the transition probabilities out of context $0$ on taking various actions.

	\begin{figure*}[!htb]
    \begin{center}
    \begin{minipage}[t]{1\textwidth}
    \begin{algorithm}[H]
	\small
	
		\caption{Estimate Transition Probabilities}
		\label{alg:estimateTransitionProbabilities}
		\begin{algorithmic}[1] 
			\State \textbf{Input:} Time budget $T'$
			\State \textbf{For} time $t\gets \{1,\dots,\frac{T'}{2}\}$ \textbf{do}
			\State \qquad Perform $do()$ at context $0$. Transition to $i\in [k]$
			\State \qquad Count number of times context $i\in[k]$ is observed
			\State \qquad Update $\widehat{q}_j^0 = \prob\left\{X_j^0 = 1\right\}$ 
			\Statex \textbf{end}
			\State Using $\widehat{q}_j^0$s, estimate $m_0$ and the set $\I_{m_o}$. Estimate $\widehat{P}_{(a,i)}=\prob[i \mid  a]$ $\enspace \forall a \in \I_{m_0}^c$ and $i\in[k]$
			\State \textbf{For} intervention $a \in \I_{m_o}$ at context $0$
			\State \qquad \textbf{For} time $t \gets \{1,\dots \frac{T'}{2 \lvert\I_{m_0}\rvert}\}$
			\State \qquad \qquad Perform $a\in\I_{m_o}$ and transition to some $i\in[k]$
			\label{step:alg2step10}
			\State \qquad \qquad Count number of times context $i$ is observed
			\Statex \qquad \textbf{end}
			\Statex \textbf{end}
			
			\State Estimate $\widehat{P}_{(a,i)}=\prob[i \mid  a]$ for each $a \in \I_{m_0}$ and contexts $i\in[k]$
			\State \textbf{return} Estimated matrix $\widehat{P}=\left[ \widehat{P}_{(a,i)}\right]_{i \in [k],a \in \I_0}$
			\footnotetext[1]{In the first half of time $T'/2$, we perform $do()$ at  State $0$.}
			\footnotetext[2]{If $\I_0:=do()\cup\{X^0_j=0,X^0_j=1\}_{j\in[n]}$, we can find $m_0\leq|\I_0|/2$ such that $\I_0 = \I_{m_0} \cup \I_{m_0}^c$ where the interventions in $\I_{m_0}^c$ are observed with probability more than $1/m_0$ and $|\I_{m_0}| = m_0$.}
			\footnotetext[3]{For the interventions $a\in\I_{m_0}^c$, we can estimate $\widehat{P}_{(a,i)}=\prob[i \mid  a]$ $\enspace \forall i\in[k]$ in the first half.}
			\footnotetext[4]{In the second half, we may intervene on the atomic interventions in $\I_{m_0}$ for time $T/(2m_0)$ each.}
			\footnotetext[5]{Using observations of $a\in\I_{m_0}$, we estimate $\widehat{P}_{(a,i)}=\prob[i \mid  a]$ $\enspace \forall a\in\I_{m_0}$ and $i\in[k]$.}
	\end{algorithmic}
	\end{algorithm}
	
	\end{minipage}
	\end{center}
    \end{figure*}
    
    \noindent Next we estimate the causal parameters at all contexts $i\in[k]$ through Algorithm \ref{alg:estimateCausalParameters}. Then we will use Algorithm \ref{alg:estimateRewards} to estimate the rewards on various interventions at the intermediate contexts. 
    
    \noindent For estimating the causal parameters, we use a variant of $\SRM$ from \citet{maiti2022causal}, which estimates the causal observational threshold $m_i$, under the setting of unobserved confounders and identifiability. We note that even in the presence of general causal graphs with hidden variables, $\SRM$ is able to efficiently estimate the rewards of all the arms simultaneously using the observational arm pulls. As mentioned in Section 3 of \citet{maiti2022causal}, the challenge is to identify the optimal number of arms with bad estimates during the initial phase of the algorithm, such that these arms can be intervened upon at the later phase. The $q_i(x)$ parameter is the minimum conditional probability of $X=x$, given different configurations of the parents of $X$. Once we have these estimates, the remaining algorithm can proceed as per usual.

    \begin{figure*}[!htb]
    \begin{center}
    \begin{minipage}[t]{1\textwidth}
    \begin{algorithm}[H]
	\small
		\caption{Estimate Causal Parameters}
		\label{alg:estimateCausalParameters}
		\begin{algorithmic}[1] 
			\State \textbf{Input:} Frequency vector $\tilde{f}$ and time budget $T'$
			\State Update $f(a) \gets \frac{1}{2}\left(\tilde{f}(a)+\frac{1}{\lvert  \I_0 \rvert}\right)\enspace \forall a\in\I_0$ 
			\State \textbf{For} intervention $a \in \I_0$
			\State \qquad\textbf{For} time $t \gets \{1,\dots T' \cdot f(a)\}$
			\State \qquad\qquad Perform $a\in \I_0$ and transition to some $i\in[k]$.
			\State \qquad\qquad At context $i$, perform $do()$ and observe $X^i_j$s
			\State \qquad\qquad Update $\widehat{q}_j^i = \min_{\text{Parents}(X_j^i),x\in\{0,1\}} \prob\left\{X_j^i = x \mid \text{Parents}(X_j^i)\right\}$ 
			\Statex \qquad\textbf{end}
			\Statex \textbf{end}
			\State Using $\widehat{q}_j^i$s, estimate $\widehat{m}_i$ values for each context $i \in [k]$
			\State \textbf{return} $\hat{M}$, the diagonal matrix of the $\widehat{m}_i$ values
			\footnotetext[1]{We choose actions $a\in\I_0$ such that we visit the contexts $i\in[k]$ approximately equally, in expectation.}
			\footnotetext[2]{On each visit to a context $i\in[k]$, we perform $do()$. From these we can estimate $q_i^j$ values, which may be used to estimate $m_i$ values.}
			\footnotetext[3]{Based on these $do()$ interventions at each context $i\in[k]$, we get estimates of $m_i$ and the intervention sets $\I_{m_i}$ such that (I) $|\I_{m_i}| = m_i$ and (II)  interventions in $\I_{m_i}$ are observed with probability less than $1/m_i$.}
		\end{algorithmic}
	\end{algorithm}
	
	\end{minipage}
	\end{center}
    \end{figure*}
        
\noindent Note that in Algorithm \ref{alg:estimateRewards} there are two phases. In the first phase, we carry out estimates for interventions that have high probability of being observed on the $do()$ intervention. In the second phase, we specifically perform interventions which have not been observed often enough. This is similar to Algorithm \ref{alg:estimateTransitionProbabilities} where we carry out the two phases of interventions at context $0$.
    
    \begin{figure*}[!htb]
    \begin{center}
    \begin{minipage}[t]{1\textwidth}
    \begin{algorithm}[H]
	\small
		\caption{Estimate Rewards}
		\label{alg:estimateRewards}
		\begin{algorithmic}[1] 
			\State \textbf{Input:} Optimal frequency $f^*$, min-max frequency $\tilde{f}$, and time budget $T'$
			\State Set $f(a) \gets \frac{1}{3}\left(f^*(a)+\tilde{f}(a)+\frac{1}{\lvert  \I_0 \rvert}\right)\enspace \forall a\in\I_0$
			\State \textbf{For} intervention $a \in \I_0$ at context $0$
			\State \qquad\textbf{For} time $t \gets \{1,\dots f(a)\cdot T'/2\}$
			\State \qquad\qquad Perform $a\in\I_0$. Transition to some $i\in[k]$. Perform $do()$ at context $i\in[k]$.
			\State \qquad\qquad Observe variables $X^i_j$'s and rewards $R_i$.
			\Statex \qquad \textbf{end}
			\Statex \textbf{end}
			\State Find the set $\I_{m_i} \enspace \forall i\in[k]$ using $q^i_j$ estimates. 
			\State Estimate mean reward $\widehat{\mathcal{R}}_{(b,i)} = \E\left[R_i  \mid  b\right]$ for each $b \in \I_{m_i}^c$
			\State \textbf{For} intervention $a \in \I_0$ at context $0$
			\State \qquad \textbf{For} time $t \gets \{1,\dots f(a)\cdot T'/2\}$
			\State \qquad\qquad Perform $a\in\I_0$ and transition to some $i\in[k]$.
			\State \qquad\qquad Iteratively perform $b\in\I_{m_i}$. Observe $R_i$
			\Statex \qquad \textbf{end}
			\Statex \textbf{end}
			\State Estimate mean reward $\widehat{\mathcal{R}}_{(b,i)} = \E\left[R_i  \mid  b\right]$ for each $b \in \I_{m_i}$ 
			\State \textbf{return} $\widehat{\mathcal{R}} = \left[ \widehat{\mathcal{R}}_{(b,i)}\right]_{i \in [k],b \in \I_i}$ 
			\footnotetext[1]{We perform the interventions in the ratio of $f^*$ which is the optimum given the $m_i$ values at the various contexts.}
			\footnotetext[2]{In the first half we estimate rewards for the interventions $\I_{m_i}^c$ in the first half, and the interventions in $\I_{m_i}$ in the second half.}
			\footnotetext[3]{Note that we round robin over the interventions $b\in\I_{m_i}$ across visits in the second half of the algorithm.}
		\end{algorithmic}
	\end{algorithm}
	
	\end{minipage}
	\end{center}
    \end{figure*}

    \clearpage

%% file: appendix/appendixProofOfUpperbound.tex
\section{Proof of Theorem 1}
	\label{appendixsection:Proof of lemmas for upperbound}

In this section, we restate Theorem \ref{theorem:main} and provide its proof, along with all the lemmas that are used in the proof.

	\begin{nonumtheorem}
		Given number of rounds $T \geq T_0$ and $\lambda$ as in equation (\ref{eqn:lambda}), \hyperref[alg:best policy generator]{\CE} achieves regret 
		$$\Regret_T \in \O\left(\sqrt{\max\left\{\frac{\lambda}{T},\frac{m_0}{Tp_+}\right\}\log\left(NT\right)}\right)$$
		
	\end{nonumtheorem}

	
	\vspace{-0.12in}
	\subsection{Proof of Theorem \ref{theorem:main}}
	\label{appendix section: proof of theorem 1}
	\vspace{-0.12in}
	To prove the theorem, we analyze the algorithm's execution as falling under either \textit{good event} or \textit{bad event}, and tackle the regret under each. 
	\begin{defn}
		\label{appendix defn: good events}
		We define five events, $E_1$ to $E_5$ (see Table \ref{appendix table: Good events table}), the intersection of which we call as \textit{good event}, $E$, i.e., $\textit{good event } E := \bigcap_{i\in[5]} E_i$. Furthermore, we define the $\textit{bad event } F:= E^c$.

    \vspace{-0.2in}
	\begin{table}[!htb]\small
	\caption{Table enumerating Good Events} \label{appendix table: Good events table}
	\begin{center}
		\begin{tabular}{|x{0.75cm}|x{4.5cm}|x{6.75cm}|}
			\hline\hline
			Event & Condition & Explanation\\
			\hline
			$E_1$& $\sum\limits_{i=1}^k \lvert \widehat{P}_{(a,i)} - P_{(a,i)}\rvert \leq \frac{ p_+ }{3} \forall a \in \I_0$ & for every intervention $a\in \I_0$, the empirical estimate of transition probability in each of \hyperref[alg:estimateTransitionProbabilities]{Algorithms \ref{alg:estimateTransitionProbabilities}},  \hyperref[alg:estimateCausalParameters]{ \ref{alg:estimateCausalParameters}} and \hyperref[alg:estimateRewards]{\ref{alg:estimateRewards}} is good, up to an absolute factor of $p_+/3$\\
			\hline
			$E_2$ & $\widehat{m}_0 \in [\frac{2}{3}m_0,2m_0]$ & our estimate for causal parameter $m_0$ for state 0 is relatively good in Algorithm \ref{alg:estimateTransitionProbabilities}.\\
			\hline
			$E_3$ & $\widehat{m}_i \in [\frac{2}{3}m_i,2m_i] \quad \forall i\in[k]$ & our estimate for causal parameter $m_i$ for each context $i\in[k]$ is relatively good in Algorithm \ref{alg:estimateCausalParameters}.\\
			\hline
			$E_4$ & $\sum_{i\in[k]} \lvert \widehat{P}_{(a,i)} - P_{(a,i)} \rvert  \leq \zeta$, $\forall a \in \I_0$ & The error in estimated transition probability in \hyperref[alg:estimateTransitionProbabilities]{Algorithm \ref{alg:estimateTransitionProbabilities}} sums to less than $\zeta$ where
			$\zeta := \sqrt{\frac{150m_0}{Tp_+}\log\left(\frac{3T}{k}\right)}$\\
			\hline
			$E_5$ & $\left\lvert \E\left[R_i  \mid  a \right] - \widehat{\mathcal{R}}_{(a,i)}\right\rvert  \leq \widehat{\eta}_i\enspace \forall i\in[k], a\in\I_i$ & The error in reward estimates in Algorithm \ref{alg:estimateRewards} is bounded\footnotemark[3] by $\widehat{\eta}_i$ where $\widehat{\eta}_i = \sqrt{\frac{27 \widehat{m}_i}{T(\widehat{P}^\tr \widehat{f}^*)_i}\log\left(2TN \right)}$\\
			
			\hline
			\hline
		\end{tabular}
	\end{center}
	
    \end{table}
    \footnotetext[3]{Recall that $\widehat{f}^*$ denotes the optimal frequency vector computed in Step \ref{step:fstar} of \hyperref[alg:best policy generator]{\CE}. Also, $(\widehat{P}^\tr \widehat{f}^*)_i$ denotes the $i$th component of the vector $P^\tr f^*$.}
    
	\end{defn}

	Considering the estimates $\widehat{P}$ and $\hat{M}$, along with frequency vector\footnotemark[2] $\widehat{f}^*$ (computed in Step \ref{step:fstar}), we define random variable 
	$$\widehat{\smash{\lambda}} := \left\lVert \widehat{P} \hat{M}^{1/2}  \left(\widehat{P}^\tr  \widehat{f}^* \right)^{\circ-\frac{1}{2}}\right\rVert_\infty^2.$$
	Note that $\widehat{\smash{\lambda}}$ is a surrogate for $\lambda$. We will show that under the good event, $\widehat{\smash{\lambda}}$ is close to $\lambda$ (Lemma \ref{appendix lem: Bounding lambda hat by lambda}).

	Recall that $\Regret_T := \E [ \varepsilon (\pi)]$ and here the expectation is with respect to the policy ${\pi}$ computed by the algorithm. We can further consider the expected sub-optimality of the algorithm and the quality of the estimates (in particular, $\widehat{P}$, $\hat{M}$ and $\widehat{\smash{\lambda}}$) under \textit{good event} (E). 
	
	Based on the estimates returned at Step \ref{step:fstar} of \hyperref[alg:best policy generator]{\CE}, either the \textit{good event} holds, or we have the \textit{bad event}. We obtain the regret guarantee by first bounding sub-optimality of policies computed under the \textit{good event}, and then bound the probability of the \textit{bad event}.
	\begin{lem}
		\label{appendix lem: difference under pi-star is small}
		For the optimal policy $\pi^*$, under the \textit{good event} ($E$), we have $\sum_{i\in[k]} P_{(\pi^*(0),i)}\E\left[R_i\  \mid \ \pi^*(i) \right]  -  \sum \widehat{{P}}_{(\pi^*(0),i)} \widehat{\mathcal{R}}_{(\pi^*(i),i)} \leq \O\left(\sqrt{\max\{\widehat{\smash{\lambda}},m_0/p_+\}/T\log\left(NT\right)}\right)
		$
	\end{lem}
	\begin{proof}
	    We add and subtract $\sum_{i\in[k]} P_{(\pi^*(0),i)}\mathcal{\widehat{R}}_{(\pi^*(i),i)}$ and reduce the expression on the left to: $\sum_{i\in[k]} P_{(\pi^*(0),i)}(\E\left[R_i  \mid  \pi^*(i) \right]-\widehat{\mathcal{R}}_{(\pi^*(i),i)}) + \sum_{i\in[k]}  \widehat{\mathcal{R}}_{(\pi^*(i),i)} (P_{(\pi^*(0),i)}-\widehat{P}_{(\pi^*(0),i)})$.
	    
		We have: (a) $\widehat{\mathcal{R}}_{(\pi^*(i),i)} \leq 1$ (as rewards are bounded) (b) $\sum_{i\in[k]} \lvert \widehat{P}_{(\pi^*(0),i)} - P_{(\pi^*(0),i)}\rvert \leq \zeta$ (by $E_4$) and (c) $\left\lvert \E\left[R_i  \mid  \pi^*(i) \right]-\widehat{\mathcal{R}}_{(\pi^*(i),i)}\right\rvert \leq \widehat{\eta}_i$ (by $E_5$). The above expression is thus bounded above by $\sum_{i\in[k]} P_{(\pi^*(0),i)}\widehat{\eta}_i + \zeta$
		Furthermore, it follows from $E_1$ (See Corollary \ref{cor: corollary to E1 as a multiplicative bound} in Section \ref{appendixsection: Bounding neg E1} in the supplementary material) that (component-wise) $P\leq \frac{3}{2}\widehat{P}$.
		Hence, the above-mentioned expression is bounded above by $\frac{3}{2}\sum_{i\in[k]}\widehat{P}_{(\pi^*(0),i)} \widehat{\eta}_i+\zeta$. Note that the definition of $\widehat{\smash{\lambda}}$ ensures $ \sum_{i\in[k]}\widehat{P}_{(\pi^*(0),i)} \widehat{\eta}_i = \O(\sqrt{{\widehat{\smash{\lambda}}}/T\log(NT)})$. Further, $\zeta = \O(\sqrt{{m_0}/{(Tp_+)}\log(T/k)})$. 
		Hence,  $ \sum_{i\in[k]}P_{(\pi^*(0),i)} \eta_i +\zeta= \O(\sqrt{\max\{{\widehat{\smash{\lambda}}},m_0/p_+\}/T \log\left(NT\right)})$, which establishes the lemma.
	\end{proof}
	We now state another similar lemma for any policy $\widehat{\pi}$ computed under \textit{good event}.  
	\begin{lem}
		\label{appendix lem: difference under widehat-pi is small}
		Let $\widehat{\pi}$ be a policy computed by \hyperref[alg:best policy generator]{\CE} under the \textit{good event} ($E$). Then, 
		$\sum_{i\in[k]}\widehat{P}_{(\widehat{\pi}(0),i)} \widehat{\mathcal{R}}_{(\widehat{\pi}(i),i)} - \sum_{i\in[k]} P_{(\widehat{\pi}(0),i)} \E\left[R_i  \mid  \widehat{\pi}(i)\right] \leq \O\left(\sqrt{\max\{\widehat{\smash{\lambda}},m_0/p_+\}/T\log\left(NT\right)}\right)$
	\end{lem}

	\begin{proof}
	We can add and subtract $\sum_{i\in[k]} P_{(\widehat{\pi}(0),i)}\mathcal{\widehat{R}}_{(\widehat{\pi}(i),i)}$ to the expression on the left to get: $\sum_{i\in[k]}  \widehat{\mathcal{R}}_{(\widehat{\pi}(i),i)} (\widehat{P}_{(\widehat{\pi}(0),i)} - P_{(\widehat{\pi}(0),i)}) + \sum_{i\in[k]} P_{(\widehat{\pi}(0),i)}(\widehat{\mathcal{R}}_{(\widehat{\pi}(i),i)}-\E\left[R_i  \mid  \widehat{\pi}(i) \right])$. Analogous to \hyperref[appendix lem: difference under pi-star is small]{Lemma \ref{appendix lem: difference under pi-star is small}}, one can show that this expression is bounded above by $\zeta + \sum_{i\in[k]}\frac{3}{2}\widehat{P}_{(\widehat{\pi}(0),i)} \widehat{\eta}_i  = \O(\sqrt{\max\{\widehat{\smash{\lambda}},m_0/p_+\}/T\log\left(NT\right)})$.
	\end{proof}
	We can also bound $\widehat{\smash{\lambda}}$ to within a constant factor of $\lambda$.
	\begin{lem}
		\label{appendix lem: Bounding lambda hat by lambda}
		Under the \textit{good event} $E$, we have $\widehat{\smash{\lambda}}\leq 8 \lambda$. 
	\end{lem}
	\begin{proof}
		Event $E_1$ ensures that $\frac{2}{3}P\leq \widehat{P}\leq \frac{4}{3}P$ (see Corollary \ref{cor: corollary to E1 as a multiplicative bound} in Appendix section \ref{appendixsection: Bounding neg E1}). In addition, note that event $E_3$ gives us $\hat{M}\leq 2M$. From these observations we obtain the desired bound: $\widehat{\smash{\lambda}} = \widehat{P}\hat{M}^{0.5}(\widehat{P}^\tr \widehat{f}^*)^{\circ-0.5} \leq \widehat{P}\hat{M}^{0.5}(\widehat{P}^\tr f^*)^{\circ-0.5} \leq 8PM^{0.5}(P^\tr f^*)^{\circ-0.5} = 8\lambda$; here, the first inequality follows from the fact that $\widehat{f}^*$ is the minimizer of the $\widehat{\smash{\lambda}}$ expression, and for the second inequality, we substitute the appropriate bounds of $\widehat{P}$ and $\hat{M}$.
	\end{proof}
    Recall that:
    \begin{align}
	    \pi^*(i) &= \argmax_{a \in \I_i} \ \E \left[ R_i \mid a\right]\vspace{-0.05in}\\
	    \vspace{-0.05in}
	    \pi^*(0) &= \argmax_{b \in\I_0} ( \sum_{i=1}^k \  \E \left[R_i\mid \pi^*(i) \right] \cdot \prob \{ i\mid b \} )
	\end{align}

    We will now define  $\varepsilon(\pi)$, denoting the sub-optimality of a policy $\pi$, as the difference between the expected rewards of $\pi^*$ and $\pi$.  
	i.e. $\varepsilon(\pi) = \sum_{i=1}^k \  \E \left[R_i\mid \pi^*(i) \right] \cdot \prob \{ i\mid \pi^*(0) \} - \sum_{i=1}^k \  \E \left[R_i\mid \pi(i) \right] \cdot \prob \{ i\mid \pi(0) \}$.
 
	\begin{cor}
		\label{appendix cor: difference between pi-star and pi-hat gives varepsilon widehat pi}
		For any $\widehat{\pi}$ computed by \hyperref[alg:best policy generator]{\CE} under \textit{good event} $E$, $\varepsilon(\widehat{\pi}) =  \O\left(\sqrt{\max\{\lambda,m_0/p_+\}/T\log\left(NT\right)}\right)$
	\end{cor}
	\begin{proof}
		Since \hyperref[alg:best policy generator]{\CE} selects the optimal policy (maximizing rewards with respect to the estimates), $\sum \widehat{P}_{(\pi^*(0),i)} \widehat{\mathcal{R}}_{(\pi^*(i),i)} \leq \sum \widehat{P}_{(\widehat{\pi}(0),i)} \widehat{\mathcal{R}}_{(\widehat{\pi}(i),i)}$. Combining this with \hyperref[appendix lem: difference under pi-star is small]{Lemmas \ref{appendix lem: difference under pi-star is small}} and \hyperref[appendix lem: difference under widehat-pi is small]{\ref{appendix lem: difference under widehat-pi is small}}, we get $\sum_{i\in[k]} P_{(\pi^*(0),i)}\E\left[R_i  \mid  \pi^*(i) \right] - \sum_{i\in[k]} P_{(\widehat{\pi}(0),i)}\E\left[R_i  \mid  \widehat{\pi}(i) \right]$ $= \O(\sqrt{\max\{{\widehat{\smash{\lambda}}},m_0/p_+\}/T\log\left(NT\right)})$ under \textit{good event}. The left-hand-side of this expression is equal to $\varepsilon(\widehat{\pi})$.
		Using Lemma \ref{appendix lem: Bounding lambda hat by lambda}, we get that $\varepsilon(\widehat{\pi}) = \O\left(\sqrt{\max\{\lambda,m_0/p_+\}/T\log\left(NT\right)}\right)$.
	\end{proof}

	\hyperref[appendix cor: difference between pi-star and pi-hat gives varepsilon widehat pi]{Corollary \ref{appendix cor: difference between pi-star and pi-hat gives varepsilon widehat pi}} shows that under the \textit{good event}, the (true) expected reward of $\pi^*$ and $\widehat{\pi}$ are within $\O\left(\sqrt{\max\{\lambda,m_0/p_+\}/T\log\left(NT\right)}\right)$ of each other.
	In Lemma \ref{lem: Bound on F} (see Section \ref{section:pr-bad-event} in the supplementary material) we will show \footnotemark[4] that $\smash{ \prob\{\bigcup_{i\in[5]}\neg E_i\} = \prob\left\{F\right\}\leq 5k/T }$ whenever $T\geq T_0$\footnotemark[5].
	
	\footnotetext[4]{Recall that, by definition, $F = E^c$.}
	\footnotetext[5]{$T_0$ as defined in Lemma \ref{lem: Bound on F} in Section \ref{section:pr-bad-event} in the supplementary material.}

	The above-mentioned bounds together establish Theorem \ref{theorem:main} (i.e., bound the regret of \hyperref[alg:best policy generator]{\CE}): $\Regret_T = \E[\varepsilon(\pi)] 
	= \E[\varepsilon(\widehat{\pi})  \mid   E ]\prob\left\{E\right\} +\E[\varepsilon(\pi')  \mid  F ]\prob\left\{F\right\}$. Since the rewards are bounded between $0$ and $1$, we have $\varepsilon(\pi') \leq 1$, for all policies $\pi'$. But $\prob\{E\}\leq 1$ giving us $\Regret_T \leq \E[\varepsilon(\pi)  \mid   E ] + \prob\{F\}$. Therefore, Corollary \ref{appendix cor: difference between pi-star and pi-hat gives varepsilon widehat pi} along with Lemma \ref{lem: Bound on F}, leads to guarantee $\Regret_T = \O\left(\sqrt{\max\{\lambda,m_0/p_+\}/T\log\left(NT\right)}\right)+5k/T = \O\left(\sqrt{\max\{\lambda,m_0/p_+\}/T\log\left(NT\right)}\right)$

%% file: appendix/appendixBoundingProbabilities.tex
\section{Bounding the Probability of the Bad Event}
	\label{appendixsection: Bounding bad events}
	
	Recall that the \textit{good event} corresponds to $\bigcap_{i\in5} E_i$ (see Definition \ref{appendix defn: good events}). Write $F := \neg\left(\bigcap_{i\in5} E_i\right)$ and note that, for the regret analysis, we require an upper bound on $\prob\{F\} = \prob\left\{\neg(\bigcap_{i\in5} E_i)\right\} = \prob\left\{\bigcup_{i\in5} \neg E_i\right\}$. Towards this, in this section we address $\prob\{\neg E_i\}$, for each of the events $E_1$-$E_5$, and then apply the union bound.
	
	\subsection[Bound on the probability that event E1 does not happen]{Bound on $\neg E_1$}
	\label{appendixsection: Bounding neg E1}
	
	The next lemma upper bounds the probability of $\neg E_1$.
	\begin{lem}
		\label{lem: Bound on E1}
		In each of Algorithms \ref{alg:estimateTransitionProbabilities}, \ref{alg:estimateCausalParameters} and \ref{alg:estimateRewards} and for all interventions $a \in \I_0$, we have 
		$\prob\{\neg E_1\} = \prob\left\{\sum\limits_{i=1}^k \lvert \widehat{P}_{(a,i)} - P_{(a,i)} \rvert >\frac{ p_+ }{3}\right\}<\frac{k}{T}$ whenever $T\geq \max\left\{\frac{1620 N }{ p_+ ^3},\frac{2025 N }{ p_+ ^2} \log\left(\frac{9NT}{k}\right)\right\}$.
	\end{lem}
	\begin{proof}
		On performing any intervention $a\in\I_0$ at context $0$, the intermediate context that we visit follows a multinomial distribution. 
		Hence, we can apply Devroye's inequality (for multinomial distributions) to obtain a concentration guarantee; we state the inequality next in our notation.  
		
			\begin{lem}[{Restatement of Lemma 3 in \cite{Devroye1983}}]
				\label{eqn: equivalent condition to luc devroye}
				Let $T_a$ be the number of times intervention $a\in\I_0$ is performed in context $0$. Then, for any $\eta >0$ and any $T_a\geq \frac{20s}{\eta^2}$, we have 
				$\prob\left\{\sum\limits_{i=1}^k \lvert \widehat{P}_{(a,i)} - P_{(a,i)} \rvert > \eta\right\}\leq 3\exp\left(-\frac{T_a\eta^2}{25}\right)$. Here, $s$ is the support of the distribution (i.e., the number of contexts that can be reached from $a$ with a nonzero probability).\\
				\end{lem}

		Note that each intervention $a\in\I_0$ is performed at least $T_a = \frac{T}{9N}$ times across Algorithms \ref{alg:estimateTransitionProbabilities}, \ref{alg:estimateCausalParameters} and \ref{alg:estimateRewards}. 
		Setting $\eta = \frac{ p_+ }{3}$ and $T_a = \frac{T}{9N}$ above, we get that for each intervention $a\in\I_0$, in each subroutine, $\prob\left\{\sum_{i=1}^k \lvert P_{(a,i)} - \widehat{P}_{(a,i)} \rvert > \frac{ p_+ }{3}\right\} \leq 3\exp\left(-\frac{T p_+ ^2}{9N \cdot  9\cdot 25}\right)=3\exp\left(-\frac{T p_+ ^2}{2025 N }\right)$.
		
		Note that to apply the inequality, we require $\frac{T}{9N}\geq \frac{180s}{ p_+ ^2}$, i.e., $T\geq \frac{1620 sN}{ p_+ ^2}$. In the current context, the support size $s$ is at most $\frac{1}{p_+}$; this follows from the fact that on performing any intervention $a\in\I_0$, at most $\frac{1}{p_+}$ contexts can have $P_{(a,i)}\geq p_+$. Hence, the requirement reduces to $T\geq \frac{1620 N}{ p_+ ^3}$.
		
		Next, we union bound the probability over the $N$ interventions (at state $0$) and the three subroutines, to obtain that, for any intervention $a \in \I_0$ and in any subroutine,  $\prob\left\{\sum_{i=1}^k \lvert P_{(a,i)} - \widehat{P}_{(a,i)} \rvert > \frac{ p_+ }{3}\right\} \leq 3N \cdot 3\exp\left(-\frac{T p_+ ^2}{2025N}\right) = 9N\exp\left(-\frac{T p_+ ^2}{2025N}\right)$.
		
		Note that $9N\exp\left(-\frac{T p_+ ^2}{2025N}\right) \leq \frac{k}{T}$, for any $T \geq \frac{2025N}{ p_+ ^2} \log\left(\frac{9NT}{k}\right)$. 
		Hence, for any $T\geq \max\left\{\frac{1620 N }{ p_+ ^3},\frac{2025 N }{ p_+ ^2} \log\left(\frac{9NT}{k}\right)\right\}$, we have $\prob[\neg E_1] \leq 9N\exp\left(-\frac{T p_+ ^2}{2025N}\right) \leq \frac{k}{T}$. This completes the proof of the lemma. 
	\end{proof}
	
	We state below a corollary which provides a multiplicative bound on $\widehat{P}$ with respect to $P$, complementing the additive form of $E_1$.
	\begin{cor}
		\label{cor: corollary to E1 as a multiplicative bound}
		Under event $E_1$, we have $\frac{2}{3}P_{(a,i)}\leq \widehat{P}_{(a,i)}\leq \frac{4}{3}P_{(a,i)}$, for all interventions $a\in\I_0$ and contexts $i\in[k]$. 
	\end{cor}
	\begin{proof}
		Event $E_1$ ensures that $\sum\limits_{i=1}^k \lvert \widehat{P}_{(a,i)} - P_{(a,i)} \rvert \leq \frac{ p_+ }{3}$, for each interventions $a\in\I_0$ and contexts $i\in[k]$. This, in particular, implies that for each intervention $a\in\I_0$ and context $i\in[k]$ the following inequality holds: $ \lvert \widehat{P}_{(a,i)} - P_{(a,i)} \rvert \leq \frac{ p_+ }{3}$. 
		Note that if $P_{(a,i)}=0$, then the algorithm will never observe context $i$ with intervention $a$, i.e., in such a case $\widehat{P}_{(a,i)} = P_{(a,i)} = 0$. 
		For the nonzero $P_{(a,i)}$s, recall that (by definition), $p_+ = \min\{P_{(a,i)}  \mid  P_{(a,i)} > 0\}$. Therefore, for any nonzero $P_{(a,i)}$, the above-mentioned inequality gives us $\lvert \widehat{P}_{(a,i)} - P_{(a,i)} \rvert \leq \frac{1}{3}P_{(a,i)}$. Equivalently, $\widehat{P}_{(a,i)}  \leq \frac{4}{3}P_{(a,i)}$ and $\widehat{P}_{(a,i)}  \geq \frac{2}{3}P_{(a,i)}$. Therefore, for all $P_{(a,i)}$s the corollary holds. 
	\end{proof}
	
	\subsection[Bound on the probability that events E2 and E3 do not happen]{Bound on Events $\neg E_2$ and $\neg E_3$}
	\label{appendixsection: Bounding neg E2 and neg E3}
	
	In this section, we bound the probabilities that our estimated $\widehat{m}_i$s are far away from the true causal parameters $m_i$s. 
	
	\begin{lem}
		\label{lem: Bound on E2}
		For any $T \geq 144m_0 \log\left(\frac{TN}{k}\right)$, in \hyperref[alg:estimateTransitionProbabilities]{Algorithm \ref{alg:estimateTransitionProbabilities}}, $\prob[\neg E_2] = \prob\left\{\widehat{m}_0 \notin [\frac{2}{3}m_0,2m_0] \right\}\leq \frac{k}{T}$.
	\end{lem}
	
	\begin{proof}
		We allocate time $\frac{T}{3}$ to \hyperref[alg:estimateTransitionProbabilities]{Algorithm \ref{alg:estimateTransitionProbabilities}}. Lemma 8 of \cite{Lattimore} ensures that, for any $\delta \in (0,1)$ and $\frac{T}{3} \geq 48 m_0 \log(\frac{N}{\delta})$, we have $\widehat{m}_0 \in [\frac{2}{3}m_0,2m_0]$, with probability at least $(1-\delta)$.    
		Setting $\delta = \frac{k}{T}$, we get the required probability bound.
	\end{proof}
	
	Next, we address $\prob\{\neg E_3  \mid  E_1\}$. 
	\begin{lem}
		\label{lem: Bound on E3}
		For any $T\geq \frac{648\max(m_i)N}{p_+} \log\left(2NT\right)$, in each of Algorithms \ref{alg:estimateCausalParameters} and \ref{alg:estimateRewards}, we have
		$\prob\left\{\exists i\in[k],\quad \widehat{m}_i \notin [\frac{2}{3}m_i,2m_i] \ \big \vert  \ E_1 \right\}\leq \frac{k}{T}$.
	\end{lem}
	\begin{proof}
		Fix any reachable context $i\in[k]$. Corresponding to such a context, there exists an intervention $\alpha\in\I_0$ such that $P_{(\alpha,i)} \geq p_+$. Event $E_1$ (Corollary \ref{cor: corollary to E1 as a multiplicative bound}) implies that $\widehat{P}_{(\alpha,i)} \geq \frac{2}{3} P_{(\alpha,i)} \geq \frac{2}{3}p_+$.
		
		Now, write $T_i$ to denote the number of times context $i\in[k]$ is visited by the Algorithms \ref{alg:estimateCausalParameters} and \ref{alg:estimateRewards}. Recall that in the subroutines we estimate $\widehat{P}_{(\alpha,i)}$ by counting the number of times context $i$ was reached and simultaneously intervention $\alpha$ observed. Furthermore, note that we allocate to every intervention at least $\frac{T}{9N}$ time (See Steps 2 in both the subroutines). In particular, intervention $\alpha$ was necessarily observed $\frac{T}{9N}$ times. Therefore, $\widehat{P}_{(a,i)} \leq \frac{T_i}{\left(\frac{T}{9N}\right)}$. This inequality leads to a useful lower bound: $T_i \geq \frac{T}{9N} \ P_{(a,i)} \geq T\frac{2p_+}{27N}$. 
		
		\CommentLines{
			
		}
		We now restate Lemma 8 from \cite{Lattimore}:
		Let $T_i$ be the number of times context $i\in[k]$ is observed. Then, 
		$\prob\left\{\widehat{m}_i\notin[\frac{2}{3}m_i,2m_i]\right\} \leq 2N\exp\left(-\frac{T_i}{48m_i}\right)$.

		\CommentLines{
			
		}
		Since $T_i \geq \frac{2Tp_+}{27N}$, this guarantee of \cite{Lattimore} corresponds to $\prob\left\{\widehat{m}_i\notin[\frac{2}{3}m_i,2m_i]\right\} \leq 2N\exp\left(-\frac{Tp_+}{648N m_i}\right) \leq 2N\exp\left(-\frac{Tp_+}{648N\max(m_i)}\right)$. 
		
		Union bounding over all contexts $i\in[k]$ and the two Algorithms \ref{alg:estimateCausalParameters} and \ref{alg:estimateRewards}, we obtain $\prob\left\{\exists i \in[k] \text{ in Algorithms \ref{alg:estimateCausalParameters}, \ref{alg:estimateRewards} } \text{ with } \widehat{m}_i \notin [\frac{2}{3}m_i,2m_i] \right\} \leq 2Nk\exp\left(-\frac{Tp_+}{648N\max(m_i)}\right)$.Finally, substituting the value of $T\geq \frac{648\max(m_i)N}{p_+} \log\left(2NT\right)$, gives us $\prob\left\{\exists i \in[k] \text{ in Algorithms \ref{alg:estimateCausalParameters}, \ref{alg:estimateRewards} } \text{ with } \widehat{m}_i \notin [\frac{2}{3}m_i,2m_i] \right\}\leq 2Nk\exp\left(-\frac{p_+}{648N\max(m_i)}\cdot \left[\frac{648\max(m_i)N}{p_+} \log\left(2NT\right)\right]\right)=\frac{k}{T}$. This completes the proof.
		\CommentLines{
			
			$\prob\left\{\forall i \in[k] \text{ in Algorithms \ref{alg:estimateCausalParameters}, \ref{alg:estimateRewards} } \frac{1}{2}\widehat{m}_i \leq m_i \leq \frac{3}{2}\widehat{m}_i \right\}\geq 1-2\sum_{i\in[k]}\delta_i = 1-2k\delta_i \geq  1- 2\sum_{i\in[k]}2N \exp\left(-\frac{Tp_+}{864N \max(m_i)}\right) = 1- 4Nk \exp\left(-\frac{Tp_+}{864N \max(m_i)}\right)$.
			
		}
	\end{proof}

	\subsection[Bound on the probability that event E4 does not happen]{Bound on $E_4$:}
	\label{appendixsection: Bounding neg E4}
	
	The following lemma provides an upper bound for $\prob\{\neg E_4 \mid  E_2\}$.
	\begin{lem}
		\label{lem: Bound on E4}
		Let $\zeta := \sqrt{\frac{150m_0}{Tp_+}\log\left(\frac{3T}{k}\right)}$. Then, 
		$\prob\{\neg E_4 \mid  E_2\} =  \prob\left\{\sum\limits_{i\in[k]}\left\lvert P_{(a,i)}-\widehat{P}_{(a,i)}\right\rvert > \zeta \big \vert  E_2 \right\} \leq \frac{k}{T}$. 
	\end{lem}
	\begin{proof}
		As in the proof of Lemma \ref{lem: Bound on E1}, we will use Devroye's inequality. Write $T_a$ to denote the number of times intervention $a\in\I_0$ is observed (in state $0$) in  Algorithm \ref{alg:estimateTransitionProbabilities}. For any $\eta \in (0,1)$ and with $T_a\geq \frac{20s}{ \eta ^2}$, Devroye's inequality gives us $\prob\left\{\sum\limits_{i=1}^k \lvert \widehat{P}_{(a,i)} - P_{(a,i)} \rvert >  \eta \right\}\leq 3\exp\left(-\frac{T_a \eta ^2}{25}\right)$. Here, $s$ is the size of the support of the multinomial distribution.
		
		We first show that $T_a$ is sufficiently large, for each intervention $a \in \I_0$. Recall that we allocate time $\frac{T}{3}$ to \hyperref[alg:estimateTransitionProbabilities]{Algorithm \ref{alg:estimateTransitionProbabilities}}. Furthermore, we observe each intervention in state $0$, at least $\frac{T}{3\widehat{m}_0}$ times, either as part of the do-nothing intervention or explicitly in Step \ref{step:alg2step10} of Algorithm \ref{alg:estimateTransitionProbabilities}. 
		Now, event $E_2$ ensures that $\widehat{m}_0 \in [\frac{2}{3}m_0,2m_0]$. Hence, each intervention $a\in\I_0$ is observed $T_a \geq \frac{T}{3\widehat{m}_0} \geq  \frac{T}{3\cdot2m_0} = \frac{T}{6m_0}$ times.
		
		Substituting this inequality for $T_a$ in the above-mentioned probability bound, we obtain \\ $\prob\left\{\sum\limits_{i=1}^k \lvert \widehat{P}_{(a,i)} - P_{(a,i)} \rvert > \eta\right\}\leq 3\exp\left(-\frac{T\eta^2}{150m_0}\right)$ when $T\geq \frac{120sm_0}{\eta^2}$. As observed in Lemma \ref{lem: Bound on E1}, the support size $s$ is at most $\frac{1}{p_+}$. Therefore, the requirement on $T$ reduces to $T\geq \frac{120m_0}{\eta^2 p_+}$. 
		

		Setting $\eta = \sqrt{\frac{150m_0}{Tp_+}\log\left(\frac{3T}{k}\right)}$ gives us   \begin{align*}\prob\left\{\sum\limits_{i=1}^k \lvert \widehat{P}_{(a,i)} - P_{(a,i)} \rvert > \sqrt{\frac{150m_0}{Tp_+}\log\left(\frac{3T}{k}\right)}\right\} &\leq 3\exp\left(\frac{-T}{150m_0}\left[\sqrt{\frac{150m_0}{Tp_+}\log\left(\frac{3T}{k}\right)}\right]^2\right) \\ 
		&\leq \frac{k}{T}.
		\end{align*} 
		Therefore $\prob\left\{\sum\limits_{i=1}^k \lvert \widehat{P}_{(a,i)} - P_{(a,i)} \rvert > \eta\right\}\leq \frac{k}{T}$, and this probability bound requires $T\geq \frac{120m_0}{\eta^2 p_+}$. That is, $\eta \geq \sqrt{\frac{120m_0}{Tp_+}}$. This inequality is satisfied by our choice of $\eta$. Hence, the lemma stands proved. 
	\end{proof}
	
	\subsection[Bound on the probability that event E5 does not happen]{Bound on $\neg E_5$}
	\label{appendixsection: Bounding neg E5}
	The next lemma bounds $\prob\{\neg E_5 \mid E_1,E_3\}$.
	\begin{lem}
		\label{lem: Bound on E5}
		Let $\widehat{\eta}_i = \sqrt{\frac{27 \widehat{m}_i}{T(\widehat{P}^\tr \widehat{f}^*)_i}\log\left(2TN \right)}$. Then, $\prob\{\neg E_5  \mid  E_3,E_1\} \leq \frac{k}{T}$. In other words: $$\prob\left\{\exists i\in[k] \text{ and } a \in \I_i \text{ such that } \left \lvert  \E\left[R_i  \mid  a\right]-\widehat{\mathcal{R}}_{(a,i)} \right \rvert > \widehat{\eta}_i \ \mid \ E_3, E_1\right\} \leq \frac{k}{T}$$. 
	\end{lem}
	\begin{proof}
		For intermediate contexts $i \in [k]$, we denote the realization of the causal parameters $m_i$ and the transition probabilities $P$ in \hyperref[alg:estimateRewards]{Algorithm \ref{alg:estimateRewards}}, as $\widetilde{m}_i$ and $\widetilde{P}$, respectively.	
		The estimates in the previous subroutines are denoted by $\widehat{m}_i$ and $\widehat{P}$.

		Event $E_1$ gives us $P_{(a,i)} \in [\frac{3}{4}\widehat{P}_{(a,i)},\frac{3}{2}\widehat{P}_{(a,i)}]$and $\widetilde{P}_{(a,i)} \in [\frac{2}{3}P_{(a,i)},\frac{4}{3} P_{(a,i)}]$. Hence, the estimates across the subroutines are close enough: $\widetilde{P}_{(a,i)} \in [\frac{1}{2}\widehat{P}_{(a,i)},2\widehat{P}_{(a,i)}]$. Similarly, event $E_3$ gives us $\widetilde{m}_i \in [\frac{1}{3}\widehat{m}_i,3\widehat{m}_i]$.
		
		Write $\widetilde{T}_i$ to denote the number of times context $i\in [k]$ was visited in \hyperref[alg:estimateRewards]{Algorithm \ref{alg:estimateRewards}}. For all contexts $i \in [k]$, we first establish a useful lower bound on $\widetilde{T}_i$, under events $E_1$ and $E_3$. The relevant observation here is that the estimate $\widetilde{P}_{(\alpha, i)}$ was computed in \hyperref[alg:estimateRewards]{Algorithm \ref{alg:estimateRewards}} by counting the number of times context $i$ was visited with intervention $\alpha \in \I_0$ (at state $0$). By construction, in \hyperref[alg:estimateRewards]{Algorithm \ref{alg:estimateRewards}} each intervention $\alpha \in \I_0$ was performed at least $\frac{\widehat{f}^*_\alpha}{3} \frac{T}{3}$ times. Furthermore, given that $\widetilde{P}_{(\alpha, i)}$ was computed via the visitation count, we get that context $i$ is visited with intervention $\alpha \in \I_0$ \emph{at least} $\widetilde{P}_{(\alpha, i)} \frac{T \widehat{f}^*_\alpha}{9}$ times. Therefore, $\widetilde{T}_i \geq \sum_{\alpha \in \I_0} \ \widetilde{P}_{(\alpha, i)} \frac{T \widehat{f}^*_\alpha}{9} = \frac{T}{9} (\widetilde{P}^\tr \widehat{f}^*)_i \geq \frac{T}{18} (\widehat{P}^\tr \widehat{f}^*)_i$. Here, the last inequality follows from the above-mentioned proximity between  $\widehat{P}$ and $\widetilde{P}$.  
		
		Now, note that, at each context $i \in [k]$, \hyperref[alg:estimateRewards]{Algorithm \ref{alg:estimateRewards}} (by construction) observes every intervention $a\in \I_i$ at least $\frac{\widetilde{T}_i}{\widetilde{m}_i}$ times. Write $\widetilde{T}_{(a,i)}$ to denote the number of times intervention $a \in \I_i$ is observed in this subroutine. Hence, 
		\begin{equation}
		    \label{equation: inequality for tildeT}
			\widetilde{T}_{(a,i)} \geq \frac{\widetilde{T}_i}{\widetilde{m}_i} \geq \frac{1}{\widetilde{m}_i} \frac{T}{18} (\widehat{P}^\tr \widehat{f}^*)_i \geq \frac{1}{3 \widehat{m}_i} \frac{T}{18} (\widehat{P}^\tr \widehat{f}^*)_i 
		\end{equation}

		For each context $i \in [k]$ and intervention $a \in \I_i$, define the event $\neg E_5(a,i)$ as 
		$\lvert \E\left[R_i  \mid  a\right]-\widehat{\mathcal{R}}_{(a,i)} \rvert >\widehat{\eta}_i$. Hoeffding's inequality gives us  $\prob\left\{ \neg E_5{(a,i)} \mid E_1, E_3 \right\}\leq 2\exp\left(-2 \widetilde{T}_{(a,i)}\widehat{\eta}_i^2\right) \leq 2\exp\left(-T \frac{(\widehat{P}^\tr \widehat{f}^*)_i \widehat{\eta}_i^2}{27 \widehat{m}_i}\right)$.
		The last inequality is obtained by substituting \hyperref[equation: inequality for tildeT]{Equation \ref{equation: inequality for tildeT}}. 
		
		Recall that $\widehat{\eta}_i = \sqrt{\frac{27 \widehat{m}_i}{T(\widehat{P}^\tr \widehat{f}^*)_i}\log\left(2TN \right)}$. Hence, the previous inequality corresponds to  $\prob\left\{ \neg E_5{(a,i)} \mid E_1, E_3 \right\} \leq 2\exp\left(-T \frac{(\widehat{P}^\tr \widehat{f}^*)_i}{27 \widehat{m}_i}\cdot \left[\sqrt{\frac{27 \widehat{m}_i}{T(\widehat{P}^\tr \widehat{f}^*)_i}\log\left(2TN \right)}\right]^2\right) = \frac{1}{TN}$. 
		
		Note that $\neg E_5 = \bigcup_{i\in[k]}\bigcup_{a\in\I_i} E_5{(a,i)}$. Taking a union bound over all contexts $i \in [k]$ and interventions $a \in \I_i$, we obtain $\prob\{\neg E_5 \mid E_1,E_3\} \leq \frac{kN}{TN} = \frac{k}{T}$.  This completes the proof. 
	\end{proof}
	\subsection{Bound on \textit{bad event} (F):}
	\label{section:pr-bad-event}
	
	Write $T_0 := \O\left(\frac{N\max(m_i)}{p_+^3}\log\left(2NT\right)\right) = \widetilde{O}\left(\frac{N\max(m_i)}{p_+^3}\right)$.
	\begin{lem}
		\label{lem: Bound on F}
		$\prob\{F\}  \leq \frac{5k}{T}$ for any $T > T_0$. 
	\end{lem}
	\begin{proof}
		\CommentLines{
			
			
			
		}
		
		We summarize the statements of Lemmas \ref{lem: Bound on E1}, \ref{lem: Bound on E2}, \ref{lem: Bound on E3}, \ref{lem: Bound on E4} and \ref{lem: Bound on E5} as follows. When $T\geq T_0$ where $T_0=  \max\left\{\frac{1620 N }{ p_+ ^3},\frac{2025 N }{ p_+ ^2} \log\left(\frac{9NT}{k}\right),144m_0 \log\left(\frac{Tn}{k}\right),\frac{864\max(m_i)N}{p_+} \log\left(2nT\right) \right\} = \O\left(\frac{N\max(m_i)}{p_+^3}\log\left(2NT\right)\right)$, we obtain
		$\prob\{F\} = \prob\left\{ \left[\bigcup_{i\in[5]} \neg E_i\right] \right\} \leq \prob\{\neg E_1\} + \prob\{\neg E_2\} + \prob\{\neg E_3 \mid  E_1 \} + \prob\{\neg E_4 \mid  E_2 \} + \prob\{\neg E_5 \mid  E_3,E_1 \} \leq  \frac{5k}{T}$.
	\end{proof}

%% file: appendix/appendixConvex.tex
\section{Nature of the Optimization Problem}
	\label{section: nature of optimization problems}
	\begin{propn}
		\label{propn: f tilde problem is an LP}
		Let $ \tilde{f} = \argmax\limits_{\text{fq.~vector} f} \enspace \min\limits_{\text{contexts [k]}}  \widehat{P}^\tr  f $. Then, finding $\tilde{f}$ is an LP
	\end{propn}
	\begin{proof}
		We rewrite the above $\max\limits_{\text{fq.~vector} f}\quad \min\limits_{i\in[k]}(\cdot)$ as a simpler program:
		\begin{align*}
			\max_{ f } \quad & z\\
			\textrm{subject to} \quad & \widehat{P}^\tr_1 f \geq z \\
			&\dots\\
			& \widehat{P}^\tr_{N} f \geq z \\
			& f \cdot \1 = 1 \\
			& f \succeq 0\\
		\end{align*}
  
		Where $N = \lvert  \I_0 \rvert$. This is equivalent to the standard form of a linear program, and hence is an LP.
	\end{proof}
	\begin{lem}
		\label{lem:optimization problem is convex}
		$\min\limits_{\text{fq.~vector} f}\quad  \max\limits_{\text{interventions } \I_0 } \widehat{P}\hat{M}^{\frac{1}{2}}\left[\widehat{P}^\tr f \right]^{\circ-\frac{1}{2}}$ is a convex optimization problem
	\end{lem}
	\begin{proof}
		First we write the $\min$-$\max$ in terms of a single minimization. First let us use the shorthand $A:= \widehat{P}\hat{M}^{\frac{1}{2}}$ and $\{A_1,\dots,A_N\}$ (where $N:=\lvert \I_0 \rvert$) denote the rows of the matrix
		\begin{align}
		  \textbf{OPT}:
			\min_{ f } \quad & z \nonumber \\
			\textrm{subject to} \quad & A_1\cdot \left[\widehat{P}^\tr f \right]^{\circ-\frac{1}{2}} \leq z  \nonumber \\
			&\dots \nonumber \\
			& A_{N}\cdot \left[\widehat{P}^\tr f \right]^{\circ-\frac{1}{2}} \leq z \label{eqn: optimization problem}\\
			& f\cdot \1 = 1 \nonumber \\
			& f\succeq 0 \nonumber
		\end{align}
		\CommentLines{
			\begin{propn}
				\label{propn: constraint equations are convex in components of f}
				The constraint equations of \hyperref[eqn: optimization problem]{OPT} are convex in the components of $f$
			\end{propn}
			\begin{proof}
				We write $f = \{f_1,\dots,f_N\}$  WLOG, we will vary the first component $f_1$. Fix $\{f_2,\dots,f_N\}$. Then note that $P^\tr f = \{\widehat{P}(*,i)^\tr f\}_{i\in[k]} = \{\widehat{P}(1,i) f_1+ c_i\}_{i\in[k]}$ for some constants $c_i,i\in[k]$.
				
				Consider one of the constraints of our problem, viz $A_1^\tr \left[\widehat{P}^\tr f \right]^{\circ-\frac{1}{2}} \leq z$. We can simplify this to get: $\sum_{i\in[k]} \frac{A_{1i}}{\sqrt{\widehat{P}(1,i) f_1+ c_i}}$. 
				
				Consider the expression: $\frac{A_{1i}}{\sqrt{\widehat{P}(1,i) f_1+ c_i}}$. This is of the form $f(x) = \frac{a}{\sqrt{bx+c}}$ for some constants $a,b,c\in\R_+$. But this is convex as $f''(x) \geq 0$.
				
				\CommentLines{We note that if $f(x) = h(g(x))$, then if $h(\cdot)$ is convex and non-increasing and $g(\cdot)$ is concave, then $f(\cdot)$ is convex. (Eq 3.10 \cite{Boyd}). We have that $h(x) = \frac{1}{x}$ is convex and non-increasing, and that $g(x) = \sqrt{bx+c}$ is concave in $x$. Thus, $f(x) = h(g(x))$ is convex. (We can also show the same by observing that $f''(x) \geq 0$).}
				
				Now note that the sum of convex functions is convex. Therefore, $\sum_{i\in[k]} \frac{A_{1i}}{\sqrt{\widehat{P}(1,i) f_1+ c_i}}$ or the first constraint is convex in $f_1$. Similarly convexity of the other constraints in $f_1$ can be shown. 
			\end{proof}
		}
		\begin{propn}
			\label{propn: ax power -0.5 is convex}
			For any  $a\in\R_+$, the function $g(x) := a x^{-\frac{1}{2}}$ is convex in $x$. 
		\end{propn}
		\begin{proof}
			We observe that the second derivative is positive.
		\end{proof}
		\begin{propn}
			The constraint equations of \hyperref[eqn: optimization problem]{OPT} are convex in $f$
		\end{propn}
		\begin{proof}
			Consider the first constraint of the problem. We can simplify this to get $\sum_{i\in[k]} \frac{A_{1i}}{\sqrt{\widehat{P}(*,i)^\tr f}}$. 
			
			Note that the $i$th term in the summand (i.e,  $\frac{A_{1i}}{\sqrt{\widehat{P}(*,i)^\tr f}}$) is of the form $f(x) = c(v^\tr x)^{-\frac{1}{2}}$ for some $c\in\R_+$ and $v \in \R^{N}_+$. Let $x_1, x_2\in\R^N$ be any two vectors, and scalar $\lambda \in [0,1]$. We wish to show that $f(\lambda x_1 + (1-\lambda)x_2) \leq \lambda f(x_1) + (1-\lambda)f(x_2)$. 
			
			We have $f(\lambda x_1 + (1-\lambda)x_2) = c(v^\tr (\lambda x_1 + (1-\lambda)x_2))^{-\frac{1}{2}} = c(\lambda v^\tr x_1 + (1-\lambda)v^\tr x_2)^{-\frac{1}{2}}$
			
			But $a x^{-\frac{1}{2}}$ is convex as per \hyperref[propn: ax power -0.5 is convex]{Proposition \ref{propn: ax power -0.5 is convex}}. Therefore $c(\lambda v^\tr x_1 + (1-\lambda)v^\tr x_2)^{-\frac{1}{2}} \leq \lambda c (v^\tr x_1)^{-\frac{1}{2}} + (1-\lambda) c (v^\tr x_2)^{-\frac{1}{2}} = \lambda f(x_1) + (1- \lambda) f(x_2)$, as required.
			
			Since $\frac{A_{1i}}{\sqrt{\widehat{P}(*,i)^\tr f}}$ is convex, the sum $\sum_{i\in[k]} \frac{A_{1i}}{\sqrt{\widehat{P}(*,i)^\tr f}}$ is convex as well. Similarly, all the other constraints are also convex.
		\end{proof}
		Since the constraints are convex in $f$ and the objective is linear, \hyperref[eqn: optimization problem]{OPT} is convex.
	\end{proof}

%% file: appendix/appendixLowerBound.tex
\section{Lower Bounds}
    \label{appendix section: analysis of lower bounds}
    This section establishes Theorem \ref{theorem: Lower Bound for our algorithm}. We will identify a collection of instances for causal bandits  with intermediate feedback and show that, for any given algorithm $\mathscr{A}$, there exists an instance in this collection for which $\mathscr{A}$'s regret is $\Omega\left(\sqrt{\frac{\lambda}{T}}\right)$.

    First we describe the collection of instances and then provide the proof.
    
	For any integer $k > 1$, consider $n=(k-1)$ causal variables at each context $i \in \{0,1,\dots, k\}$. The transition matrix $P$ is set to be deterministic. Specifically, for each $i \in [n]$, we have $\prob\{i \mid do(X_i^0 = 1)\} = 1$. For all other interventions at context 0, we transition to context k with probability 1. Such a transition matrix can be achieved by setting $q^0_i=0$ for all $i \in [k-1]$. As before, the total number of interventions $N:=2n+1 = 2k-1$.
	
	Now consider a family of $Nk+1$ instances\footnotemark[6] $\left\{\mathcal{F}_0\right\} \cup  \left\{ \mathcal{F}_{(a,i)} \right\}_{i \in [k], a \in \I_i}$. Here, $\mathcal{F}_0$ and each $\mathcal{F}_{(a,i)}$ is an instance of a causal bandit with intermediate feedback with the above-mentioned transition probabilities. The instances differ in the rewards at the intermediate contexts. In particular, in instance $\mathcal{F}_0$, we set the reward distributions such that  $\E[R_i \mid a]=\frac{1}{2}$ for all contexts $i \in [k]$ and interventions $a \in \I_i$. For each $i \in [k]$ and $a \in \I_i$, instance $\mathcal{F}_{(a,i)}$ differs from $\mathcal{F}_0$ only at context $i$ and for intervention $a$. Specifically, by construction, we will have $\E[R_i \mid a] = \frac{1}{2} + \beta$, for a parameter $\beta >0$. The expected rewards under all other interventions will be $1/2$, the same as in $\mathcal{F}_0$.
	\footnotetext[6]{Note the change in notation. We used the term $\mathcal{F}_{i,j}$ instead of $\mathcal{F}_{(a,i)}$ in the main paper. This has been amended in a later version of the main paper.}
	
	Given any algorithm $\mathscr{A}$, we will consider the execution of $\mathscr{A}$ over all the instances in the family. The execution of algorithm $\mathscr{A}$ over each instance induces a trace, which may include the realized transition probabilities $\widehat{P}$, the realized variable probabilities $\widehat{q}^i_j$ for $i\in [k]$ and $j \in [n]$ and the corresponding $\widehat{m}_i$s, and the realized rewards $\widehat{\mathcal{R}}$. Each of such realizations (random variables) has a corresponding distribution (over many possible runs of the algorithm). We call the measures corresponding to these random variables under the instances $\mathcal{F}_0$ and $\mathcal{F}_{(a,i)}$ as $\mathcal{P}_0$ and $\mathcal{P}_{(a,i)}$, respectively.

	\subsection{Proof of Theorem \ref{theorem: Lower Bound for our algorithm}}
	For any algorithm $\mathscr{A}$ and given time budget $T$, we first consider the $\mathscr{A}$'s execution over instance $\mathcal{F}_0$. As mentioned previously, $\mathcal{P}_0$ denotes the trace distribution induced by the algorithm for $\mathcal{F}_0$. In particular, write $r_i$ to denote the expected number of times context $i$ is visited,  $r_i := \E_{\mathcal{P}_0}\left[\text{state $i$ is visited}\right]/T$. 
	
	

    Recall that $m_i := \max\{j  \mid  q^i_{(j)}<\frac{1}{j}\}$ and $\I_{m_i} := \{do(X^i_{(j)}=1)  \mid  q^i_{(j)}<\frac{1}{j}\}$, where the Bernoulli probabilities of the variables at context $i$ are sorted to satisfy $q^i_{(1)}\leq q^i_{(2)}\leq \dots \leq q^i_{(n)}$. Note that these definitions do not depend on the algorithm at hand. The algorithm, however, may choose to perform different interventions different number of times. Write $N_{(a,i)}$ to denote the expected (under $\mathcal{P}_0$) number of times intervention $a$ is performed by the algorithm at context $i$. Furthermore, let random variable $T_{(a,i)}$ denote the number of times intervention $a$ is observed at context $i$. Hence, $\E_{\mathcal{P}_0}[T_{(a,i)}]$ is the expected number of times intervention $a$ is observed\footnotemark[7].
    
    \footnotetext[7]{Note that $a$ can be observed while performing the do-nothing intervention. Also, the expected value $N_{(a,i)}$ accounts for the number of times $a$ is explicitly performed and not just observed.}

    Using the expected values for algorithm $\mathscr{A}$ and instance $\mathcal{F}_0$, we define a subset of $\I_{m_i}$ as follows: $\J_i := \left\{a\in \I_{m_i}  \ : \  N_{(a,i)} \leq 2\frac{Tr_i}{m_i} \right\}$. The following proposition shows that the size of $\J_i$ is sufficiently large. 
    \begin{propn}
    \label{propn: Size of Ai}
        The set $\J_i$ is non-empty. In particular, 
        \begin{align*}
            m_i/2\leq \lvert \J_i \rvert \leq m_i.
        \end{align*}
    \end{propn}
    \begin{proof}
    The upper bound on the size of subset $\J_i$ follows directly from its definition: since $\J_i\subseteq I_{m_i}$ we have $\lvert \J_i \rvert \leq \lvert \I_{m_i} \rvert = m_i$. 
    
    For the lower bound on the size of $\J_i$, note that $T r_i$ is the expected number of times context $i$ is visited by the algorithm. Therefore, 
    \begin{align}
        \sum_{a \in \I_{m_i}} N_{(a,i)} \leq T r_i \label{ineq:n-avg}
    \end{align} 
    Furthermore, by definition, for each intervention $b \in \I_{m_i} \setminus \J_i$ we have $N_{(b,i)} \geq \frac{2 T r_i}{m_i}$. Hence, assuming $\lvert \I_{m_i} \setminus \J_i \rvert > \frac{m_i}{2}$ would contradict inequality (\ref{ineq:n-avg}). This observation implies that $\lvert \I_{m_i} \setminus \J_i \rvert \leq \frac{m_i}{2}$ and, hence, $\lvert \J_i \rvert \geq \frac{m_i}{2}$. This completes the proof.  
    \end{proof}
    
    Recall that $T_{(a,i)}$ denotes the number of times intervention $a$ is observed at context $i$. The following proposition bounds $\E[T_{(a,i)}]$ for each intervention $a \in \J_i$. 
    \begin{propn}
    \label{propn: Number of interventions of elements in Ai}
    For every intervention $a \in \J_i$
    \begin{align*}
        \E_{\mathcal{P}_0}[T_{(a,i)}] \leq \frac{3Tr_i}{m_i}.
    \end{align*}
    \end{propn}
    \begin{proof}
        Any intervention $a\in\J_i\subseteq \I_{m_i}$ may be observed either when it is explicitly performed by the algorithm or as a random realization (under some other intervention, including do-nothing). Since $a \in \I_{m_i}$, the probability that $a$ is observed as part of some other intervention is at most $\frac{1}{m_i}$. Therefore, the expected number of times that $a$ is observed by the algorithm---without explicitly performing it---is at most $\frac{T r_i}{m_i}$; \footnotemark[7] recall that the expected number of times context $i$ is visited is equal to $T r_i$. 
        \footnotetext[7]{Here, we use the fact that the realization of $a$ is independent of the visitation of context $i$.}
    
    For any intervention $a \in \J_i$, by definition, the expected number of times $a$ is performed $N_{(a,i)} \leq \frac{2 T r_i}{m_i}$. Therefore, the proposition follows: 
        \begin{align*}
            \E[T_{(a,i)}] \leq \frac{T r_i}{m_i} + N_{(a,i)} \leq \frac{3Tr_i}{m_i}.
        \end{align*}
    \end{proof}

    We now state two known results for KL divergence. 
    
			    \textbf{Bretagnolle-Huber Inequality (Theorem 14.2 in \cite{LattimoreBook})}
				\label{eqn: Bretagnolle-Huber Inequality}: 
				Let $\mathcal{P}$ and $\mathcal{P}'$ be any two measures on the same measurable space. Let $E$ be any event in the sample space with complement $E^c$. Then, 
				\begin{align} 
				\prob_{\P}\{E\} + \prob_{\mathcal{P}'}\{E^c\}\geq\frac{1}{2}\exp\left(-\rm{KL}(\mathcal{P},\mathcal{P}')\right). 
				\end{align}
	
			    \textbf{Bound on KL-Divergence with number of observations (Adaptation of Equation 17 in Lemma B1 from \cite{AuerGamblinginaRiggedCasino})}:
				\label{eqn: Bound on KL Divergence with number of observations}
				Let $\mathcal{P}_0$ and $\mathcal{P}_{(a,i)}$ be any two measures with differing  expected rewards (for exactly the intervention $a$ at context $i$) by an amount $\beta$. Then, 
				\begin{align}
				    \rm{KL}(\mathcal{P}_0,\mathcal{P}_{(a,i)}) \leq 6\beta^2 \  \E_{\mathcal{P}_0}[T_{(a,i)}] \label{ineq:auer}
				\end{align}
				
	
	Using \hyperref[eqn: Bound on KL Divergence with number of observations]{this bound on KL divergence} and \hyperref[propn: Number of interventions of elements in Ai]{Proposition \ref{propn: Number of interventions of elements in Ai}}, we have, for all contexts $i\in [k]$ and interventions $a \in \J_i$:  
	\begin{align}
	    \rm{KL}(\mathcal{P}_0,\mathcal{P}_{(a,i)}) \leq 6\beta^2\cdot 3\frac{Tr_i}{m_i} = 18\frac{Tr_i\beta^2}{m_i}
	\end{align}
	
	Substituting this in the \hyperref[eqn: Bretagnolle-Huber Inequality]{Bretagnolle-Huber Inequality}, we obtain, for any event $E$ in the sample space along with all contexts $i \in [k]$ and all interventions $a \in \J_i$: 
    \begin{align}
    \label{eqn: bound on Prob E under P1 and Prob Ec under P2 using KL}
        \prob_{\P_{(a,i)}}\{E\} + \prob_{\P_0}\{E^c\} \geq\frac{1}{2}\exp\left(-18\frac{Tr_i\beta^2}{m_i}\right)
    \end{align}
    
    
    We now define events to lower bound the probability that Algorithm $\mathscr{A}$ returns a sub-optimal policy. In particular, write $\widehat{\pi}$ to denote the policy returned by algorithm $\mathscr{A}$. Note that $\widehat{\pi}$ is a random variable. 
    
    For any $\ell \in [k]$ and any intervention $b$, write $G_1(b, \ell)$ to denote the event that---under the returned policy $\widehat{\pi}$---intervention $b$ is \emph{not} chosen at context $\ell$, i.e., $G_1(b, \ell) := \left\{\widehat{\pi}(\ell)\neq b\right\}$. Also, let $G_2(\ell)$ denote the event that policy $\widehat{\pi}$ does not induce a transition to $\ell$ from context $0$, i.e., $G_2(\ell) := \left\{\widehat{\pi}(0)\neq \ell\right\}$. Furthermore, write $G (b, \ell) := G_1(b, \ell) \cup G_2 (\ell)$. Note that the complement $G^c(b, \ell) = G^c_1(b, \ell) \cap G^c_2(\ell) = \{ \widehat{\pi}(\ell) = b \} \cap \{ \widehat{\pi}(0) = \ell \}$.
 
    Considering measure $\mathcal{P}_0$, we note that for each context $\ell \in [k]$ there exists an intervention $\alpha_\ell \in \J_\ell$ with the property that $\prob_{\mathcal{P}_0} \left\{ G^c_1(\alpha_\ell, \ell) \right\} = \prob_{\mathcal{P}_0} \left\{ \widehat{\pi}(\ell) = \alpha_\ell \right\} \leq \frac{1}{\lvert J_\ell \rvert}$. This follows from the fact that $\sum_{a \in \J_\ell}\prob_{\mathcal{P}_0}\left\{\widehat{\pi}(\ell)= a\right\} \leq 1$. Therefore, for each context $\ell \in [k]$ there exists an intervention $\alpha_\ell$ such that $\prob_{\P_0}\{ G^c(\alpha_\ell, \ell) \} \leq \frac{1}{\lvert \J_\ell \rvert}$. 

    This bound and \hyperref[eqn: bound on Prob E under P1 and Prob Ec under P2 using KL]{inequality \ref{eqn: bound on Prob E under P1 and Prob Ec under P2 using KL}} imply that for all contexts $\ell \in [k]$ there exists an intervention $\alpha_\ell$ that satisfies 
    \begin{align}
    \label{eqn: forall i exists a in Ji such that prob of making mistake is bouded below using KL}
      \prob_{\P_{(\alpha_\ell,\ell)}}\{G(\alpha_\ell, \ell) \} \geq \frac{1}{2}\exp\left(-18\frac{Tr_\ell \beta^2}{m_\ell}\right) - \frac{1}{\lvert \J_\ell \rvert}
    \end{align}
    
    We will set 
    \begin{align}
    \label{eqn: Beta Value}
        \beta = \min\left\{\frac{1}{3},\sqrt{\frac{\sum_{\ell\in[k] } m_\ell }{18T}}\right\}
    \end{align}
    
    Therefore $\beta$ takes value either $\sqrt{\frac{\sum_{\ell\in[k] } m_\ell }{18T}}$ or $\frac{1}{3}$. We will address these over two separate cases.
    
    \textbf{Case 1}: 
    $\beta = \sqrt{\frac{\sum_{\ell\in[k] } m_\ell }{18T}}$.
    
    We wish to substitute this $\beta$ value in 
    \hyperref[eqn: forall i exists a in Ji such that prob of making mistake is bouded below using KL]{Equation \ref{eqn: forall i exists a in Ji such that prob of making mistake is bouded below using KL}}. Towards this, we will state a proposition.
    
    \begin{propn}
    \label{propn: min-max for s and alpha s}
    There exists a context $s\in [k]$ such that $$\sqrt{\frac{m_s}{18Tr_s}}\geq \sqrt{\frac{\sum_{\ell\in[k] } m_\ell }{18T}}$$
    \end{propn}
    \begin{proof}
    First, we note the following claim considering all vectors $\rho=\{\rho_1,\dots,\rho_k\}$ in the probability simplex $\Delta$.
    \begin{claim}
    \label{claim: inequality for sum mi-s}
    For any given set of integers $m_1, m_2, \ldots, m_k$, we have 
    \begin{align*}
        \min_{(\rho_1, \rho_2, \ldots, \rho_k) \in\Delta} \ \left( \max_{\ell\in[k]} \frac{m_\ell}{\rho_\ell} \right) \geq \sum_{\ell\in[k]} m_\ell
    \end{align*}
    \end{claim}
    \begin{proof}
        Assume, towards a contradiction, that for all $\ell\in [k]$, we have $\frac{m_\ell}{\rho_\ell} < \sum_{\ell\in[k]} m_\ell$. Then, $\rho_\ell > \frac{m_\ell}{\sum_{\ell\in[k]} m_\ell}$, for all $\ell \in [k]$. Therefore, $\sum_{\ell\in[k]} \rho_\ell > \sum_{\ell\in[k]}\frac{m_\ell}{\sum_{\ell\in[k]} m_\ell} = 1$. However, this is a contradiction as $\sum_{\ell\in[k]} \rho_\ell = 1$.
    \end{proof}
    
    An immediate consequence of \hyperref[claim: inequality for sum mi-s]{Claim \ref{claim: inequality for sum mi-s}} is that $$\min_{(r_1, r_2, \ldots, r_k) \in\Delta} \left( \max_{\ell\in[k]}  \sqrt{\frac{m_\ell}{18Tr_\ell}}\right) \geq \sqrt{\frac{\sum_{\ell \in [k]} m_\ell}{18T}}$$.

    Therefore, irrespective of how $r_i$s are chosen, there always exists a context $s\in[k]$ such that $\sqrt{\frac{m_s}{18Tr_s}}\geq \sqrt{\frac{\sum_{\ell \in [k]} m_\ell}{18T}}$. 
    \end{proof}
    
    For such a context $s\in[k]$ that satisfies \hyperref[propn: min-max for s and alpha s]{Proposition \ref{propn: min-max for s and alpha s}}, we note that, $\frac{m_s}{18Tr_s}\geq \beta^2$ or $\frac{18Tr_s\beta^2}{m_s}\leq 1$.
    
    Let us now restate \hyperref[eqn: forall i exists a in Ji such that prob of making mistake is bouded below using KL]{Equation \ref{eqn: forall i exists a in Ji such that prob of making mistake is bouded below using KL}} for such a context $s$.
    There exists a context $s \in [k]$ and an intervention $\alpha_s$ that satisfies 
    \begin{align}
    \label{eqn: exists s exists a-s in Ji such that prob of making mistake is bouded below using KL}
      \prob_{\P_{(\alpha_s,s)}}\{G(\alpha_s, s) \} \geq \frac{1}{2}\exp\left(-18\frac{Tr_s \beta^2}{m_s}\right) - \frac{1}{\lvert \J_s \rvert} \geq \frac{1}{2e} - \frac{1}{\lvert \J_s \rvert}
    \end{align}
    
    Note that the last inequality lower bounds the to probability of selecting a non-optimal policy when the algorithm $\mathscr{A}$ is executed on instance $\mathcal{F}_{\alpha_s, s}$.
    Furthermore, in instance $\mathcal{F}_{\alpha_s, s}$, for any non-optimal policy $\widehat{\pi}$ we have $\varepsilon(\widehat{\pi}) = \left( \frac{1}{2} + \beta \right) - \frac{1}{2} = \beta$. Therefore, we can lower bound $\mathscr{A}$'s regret over instance $\mathcal{F}_{\alpha_s, s}$ as follows: 
    \begin{align}
        \Regret_T = \E[\varepsilon(\widehat{\pi})]&= \prob_{\P_{(\alpha_s,s)}}\{G (\alpha_s, s) \} \cdot \E[\Regret \mid G(\alpha_s, s)]\enspace +\enspace\\&\qquad\qquad \prob_{\P_{(\alpha_s,s)}}\{G^c(\alpha_s, s)\} \cdot \E[\Regret \mid G^c(\alpha_s, s)]\nonumber\\
        &\geq \left[\frac{1}{2e} - \frac{1}{\lvert \J_s \rvert}\right] \beta +\enspace \prob_{\P_{(\alpha_s, s)}}\{G^c(\alpha_s, s)\}\cdot 0\nonumber\\
        &= \left[\frac{1}{2e} - \frac{1}{\lvert \J_s \rvert}\right] \beta
        \label{eqn: regret in terms of e Js and beta}
    \end{align}
    Note that we can construct the instances to ensure that $m_\ell \geq 8$, for all contexts $\ell$, and, hence, $\left(\frac{1}{2e} - \frac{1}{\lvert \J_i \rvert} \right) = \Omega(1)$ (see Proposition \ref{propn: Size of Ai}). Therefore \hyperref[eqn: regret in terms of e Js and beta]{Equation \ref{eqn: regret in terms of e Js and beta}} gives us:
    \begin{align}
        \Regret_T&= \Omega (\beta) = \Omega \left(\sqrt{\frac{\sum_{\ell\in[k]}m_\ell}{T}}\right) \label{ineq: regret case 1}
    \end{align}


    

    
    \textbf{Case 2} We now consider the case when $\beta = \frac{1}{3}$. In such a case, $\sqrt{\frac{\sum_{\ell\in[k] } m_\ell }{18T}}> \frac{1}{3}$. 
    
    We showed in \hyperref[propn: min-max for s and alpha s]{Proposition \ref{propn: min-max for s and alpha s}} that there exists a context $s\in [k]$ such that $\sqrt{\frac{m_s}{18Tr_s}}\geq \sqrt{\frac{\sum_{\ell\in[k] } m_\ell }{18T}}$. Combining the two statements, there exists a context $s$ such that $\sqrt{\frac{m_s}{18Tr_s}}\geq \frac{1}{3}$. 
    We now restate Inequality \ref{eqn: forall i exists a in Ji such that prob of making mistake is bouded below using KL} for such a context $s\in[k]$: $$\prob_{\P_{(\alpha_s,s)}}\{G(\alpha_s, s) \} \geq \frac{1}{2}\exp\left(-9 \beta^2\right) - \frac{1}{\lvert \J_s \rvert} = \frac{1}{2e} - \frac{1}{\lvert \J_s \rvert}$$
    
    Following the exact same procedure as in Case 1, we can derive that $\Regret_T \geq \left[\frac{1}{2e} - \frac{1}{\lvert \J_s \rvert}\right] \beta$. We saw in Case 1 that it is possible to construct instances such that $\left[\frac{1}{2e} - \frac{1}{\lvert \J_s \rvert}\right] = \Omega(1)$. Therefore the following holds for Case 2 also: 
    \begin{align}
        \Regret_T&= \Omega (\beta) = \Omega \left(\sqrt{\frac{\sum_{\ell\in[k]}m_\ell}{T}}\right) \label{ineq: regret case 2}
    \end{align}

    Inequalities \ref{ineq: regret case 1} and \ref{ineq: regret case 2} imply that there exists a context $s$ and an intervention $\alpha_{s}$ such that, under instance  $\mathcal{F}_{( \alpha_s,s)}$, algorithm $\mathscr{A}$'s regret satisfies  
    \begin{align}
    \label{eqn: Regret in terms of sum mi}
        \Regret_T = \Omega \left(\sqrt{\frac{\sum_{\ell\in[k]}m_\ell}{T}}\right)
    \end{align}
    
    We complete the proof of Theorem \ref{theorem: Lower Bound for our algorithm} by showing that in the current context $\lambda = \sum_{\ell\in[k]} m_\ell$. 
    
    \begin{propn}
    \label{propn: Value of Lambda for chosen transition probability matrix}
    For the chosen transition matrix  
    $$\lambda := \min_{\text{fq.~vector} f} \ \left\lVert PM^{1/2}  \left(P^\tr  f \right)^{\circ-\frac{1}{2}} \right\rVert_\infty^2  = \sum_{\ell\in[k]}m_\ell$$
    \end{propn}
    \begin{proof}
    Recall that all the instances, $\mathcal{F}_0$ and $\mathcal{F}_{(a,i)}$s, have the same (deterministic) transition matrix $P$. Also, parameter $\lambda$  is computed via  \hyperref[eqn:lambda]{Equation \ref{eqn:lambda}}. 
    
    Consider any frequency vector $f$ over the interventions $\{1,\dots,N\}$. From the chosen transition matrix, we have the following:

    \begin{align*}
    P = 
    \begin{bmatrix}
    1 & 0 & \dots & 0 \\
    0 & 1 & \dots & 0 \\
    &&\dots\\
    0 & 0 & \dots & 1 \\
    &&\dots\\
    0 & 0 & \dots & 1 \\
    \end{bmatrix}& \quad
    PM^{\frac{1}{2}} = 
    \begin{bmatrix}
    \sqrt{m_1} & 0 & \dots & 0 \\
    0 & \sqrt{m_2} & \dots & 0 \\
    &&\dots\\
    0 & 0 & \dots & \sqrt{m_k} \\
    &&\dots\\
    0 & 0 & \dots & \sqrt{m_k} \\
    \end{bmatrix}& \quad
    P^\tr f = 
    \begin{bmatrix}
    f_1\\ f_2\\ \dots\\ f_{k-1}\\ f_k + \ldots + f_N\\
    \end{bmatrix}
    \end{align*}    
    
    From here, we can compute the following:
    \begin{align*}
      PM^{1/2}  \left(P^\tr  f \right)^{\circ-\frac{1}{2}} =
      \left[
      \sqrt{\frac{m_1}{f_1}},  \dots,  \sqrt{\frac{m_{k-1}}{f_{k-1}}},  \sqrt{\frac{m_k}{f_k + \ldots + f_N}},  \dots,  \sqrt{\frac{m_k}{f_k + \ldots + f_N}}
      \right]^\tr
    \end{align*}
    
    That is, for all $\ell \in [k-1]$, the $\ell$th component of the vector $PM^{1/2}  \left(P^\tr  f \right)^{\circ-\frac{1}{2}}$ is equal to $\sqrt{\frac{m_i}{f_i}}$. All the remaining components are $\sqrt{\frac{m_k}{f_k + \ldots + f_N}}$. 
    
    Write $\rho_\ell := f_\ell$ for all $\ell \in [k-1]$ and $\rho_k = \sum_{j=k}^N f_j$. Since $f$ is a frequency vector, $(\rho_1, \ldots \rho_k) \in \Delta$. In addition,  
    $$PM^{1/2}  \left(P^\tr  f \right)^{\circ-\frac{1}{2}} =\left[\sqrt{\frac{m_1}{\rho_1}},\dots ,\sqrt{\frac{m_{k-1}}{\rho_{k-1}}},\sqrt{\frac{m_k}{\rho_k}},\dots,\sqrt{\frac{m_k}{\rho_k}}\right]^\tr$$
    Therefore, by definition, $\lambda =\min_{(\rho_1, \ldots, \rho_k) \in\Delta} \left( \max_{\ell\in[k]} {\frac{m_\ell}{\rho_\ell}} \right)$. Now, using a complementary form of \hyperref[claim: inequality for sum mi-s]{Claim \ref{claim: inequality for sum mi-s}} we obtain $\lambda = {\sum_{\ell\in[k]}m_\ell}$. The proposition stands proved. 
    
    \end{proof}
    
    Finally, substituting \hyperref[propn: Value of Lambda for chosen transition probability matrix]{Proposition \ref{propn: Value of Lambda for chosen transition probability matrix}} into \hyperref[eqn: Regret in terms of sum mi]{Equation \ref{eqn: Regret in terms of sum mi}}, we obtain that there exists an instance $\mathcal{F}_{(\alpha_s, s)}$ for which algorithm $\mathscr{A}$'s regret is lower bounded as follows 
    \begin{align}
        \Regret_T = \Omega \left(\sqrt{\frac{\lambda}{T}}\right).
    \end{align}
This completes the proof of Theorem \ref{theorem: Lower Bound for our algorithm}.	
	
	\subsection{Proof of Inequality (\ref{ineq:auer})}
For completeness, we provide a proof of inequality (\ref{ineq:auer}). 
	
	\begin{lem}
		\label{lem: KL leq 6 epsilon square times number of obs}
		$ \rm{KL} (\P_0,\P_{(a,i)}) \leq 6\beta_i^2 \ \E_{\P_0} [ T_{(a,i)}]$ 
	\end{lem}
	\begin{proof}[Proof of Inequality (\ref{ineq:auer})]
	\label{proof: proposition KL leq 6 epsilon square times number of obs}
		This proof is based on lemma B1 in \cite{AuerGamblinginaRiggedCasino}. We define a couple of notations for this proof.
		Let $\mathbf{R}_{t-1}$ indicate the filtration (of rewards and other observations) up to time $t-1$. and $R_t$ indicate the reward at time $t$ for this proof.
		\begin{align*}
			 \rm{KL} (\P_0,\P_{(a,i)}) &=  \rm{KL} \left[\prob_{\P_0}(R_T,R_{T-1},\dots,R_1) \mathrel{\Vert} \prob_{\P_{(a,i)}}(R_T,R_{T-1},\dots,R_1)\right]
		\end{align*}
		We now state (without proof) a useful lemma for bounding the KL divergence between random variables over a number of observations.
		
    			\textbf{Chain Rule for entropy (Theorem 2.5.1 in \cite{JoyCoverBook})}:
    			\label{eqn: Chain rule for entropy}
    			Let $X_1,\dots,X_T$ be random variables drawn according to $P_1,\dots,P_T$. Then 
    			$$H(X_1,X_2,\dots,X_T) = \sum_{i=1}^T H(X_i\mid X_{i-1},X_{i-2},\dots,X_1)$$
    			where $H(\cdot)$ is the entropy associated with the random variables.
    	
		Using the \hyperref[eqn: Chain rule for entropy]{chain rule for entropy}
		\begin{align*}
			 \rm{KL} (\P_0,\P_{(a,i)}) &=\sum\limits_{t=1}^T  \rm{KL} \left[\prob_{\P_0}(R_t \mid \mathbf{R}_{t-1}) \mathrel{\Vert} \prob_{\P_{(a,i)}}(R_t \mid \mathbf{R}_{t-1})\right]\\
			\intertext{Let $a_t$ be the intervention chosen by the Algorithm $\mathscr{A}$ at time $t$. Then:}
			&=\sum\limits_{t=1}^T \prob_{\P_0}\{a_t \neq a \mid \mathbf{R}_{t-1}\} \left(\frac{1}{2} \mathrel{\Vert} \frac{1}{2}\right) + \prob_{\P_0}\{a_t = a \mid \mathbf{R}_{t-1}\} \rm{KL} \left(\frac{1}{2} \mathrel{\Vert} \frac{1}{2} + \beta_i\right) \\
			\intertext{Since $ \rm{KL} \left(\frac{1}{2} \mathrel{\Vert} \frac{1}{2}\right) = 0$, we get:}
			&=\sum\limits_{t=1}^T \prob_{\P_0}\{a_t = a \mid \mathbf{R}_{t-1}\} \rm{KL} \left(\frac{1}{2} \mathrel{\Vert} \frac{1}{2} + \beta_i\right) \\
			&=  \rm{KL} \left(\frac{1}{2} \mathrel{\Vert} \frac{1}{2} + \beta_i\right) \sum\limits_{t=1}^T \prob_{\P_0}\{a_t = a \mid \mathbf{R}_{t-1}\} \\
			&=  \rm{KL} \left(\frac{1}{2} \mathrel{\Vert} \frac{1}{2} + \beta_i\right) \E_{\P_0}[T_{(a,i)}] \\
			%
			%
		\end{align*}

		\begin{claim}
			\label{claim: KL leq 6 betai square}
			$ \rm{KL} \left(\frac{1}{2} \mathrel{\Vert} \frac{1}{2}+\beta_i\right)=-\frac{1}{2} \log_2(1-4\beta_i^2)\leq 6\beta_i^2$
		\end{claim}
		\begin{proof}
			\begin{align*}
				 \rm{KL} \left(\frac{1}{2} \mathrel{\Vert} \frac{1}{2}+\beta_i\right) &= \frac{1}{2}\log_2\left[\frac{\frac{1}{2}}{\frac{1}{2} + \beta_i}\right] + (1-\frac{1}{2})\log_2\left[\frac{(1-\frac{1}{2})}{(1-\frac{1}{2} - \beta_i)}\right]\\
				&= \frac{1}{2}\log_2\left[\frac{1}{1 + 2\beta_i}\right] + \frac{1}{2}\log_2\left[\frac{1}{1 - 2\beta_i}\right] \\&= \frac{1}{2}\log_2\left[\frac{1}{1 - 4\beta_i^2}\right]=-\frac{1}{2}\log_2\left[1 - 4\beta_i^2\right]\\
				&=-\frac{1}{2\ln(2)} \ln\left[1 - 4\beta_i^2\right] \leq \frac{4\beta_i^2}{2\ln(2)} < 6\beta_i^2
			\end{align*}
			where the last inequality is obtained from the Taylor series expansion of the $\log$.
		\end{proof}
		It follows that: 
		$ \rm{KL} (\prob_0,\prob_1) \leq 6\beta_i^2\E_{\P_0}[T_{(a,i)}]$.
	\end{proof}

%% file: main.bbl
\begin{thebibliography}{55}
\providecommand{\natexlab}[1]{#1}
\providecommand{\url}[1]{\texttt{#1}}
\expandafter\ifx\csname urlstyle\endcsname\relax
  \providecommand{\doi}[1]{doi: #1}\else
  \providecommand{\doi}{doi: \begingroup \urlstyle{rm}\Url}\fi

\bibitem[Acharya et~al.(2018)Acharya, Bhattacharyya, Daskalakis, and Kandasamy]{acharya2018learning}
Jayadev Acharya, Arnab Bhattacharyya, Constantinos Daskalakis, and Saravanan Kandasamy.
\newblock Learning and testing causal models with interventions.
\newblock \emph{Advances in Neural Information Processing Systems}, 31, 2018.

\bibitem[Agrawal \& Goyal(2012)Agrawal and Goyal]{agrawal2012analysis}
Shipra Agrawal and Navin Goyal.
\newblock Analysis of thompson sampling for the multi-armed bandit problem.
\newblock In \emph{Conference on learning theory}, pp.\  39--1. JMLR Workshop and Conference Proceedings, 2012.

\bibitem[Ali et~al.(2005)Ali, Richardson, Spirtes, and Zhang]{ali2005towards}
R~Ayesha Ali, Thomas~S Richardson, Peter Spirtes, and Jiji Zhang.
\newblock Towards characterizing markov equivalence classes for directed acyclic graphs with latent variables.
\newblock In \emph{Proceedings of the Twenty-First Conference on Uncertainty in Artificial Intelligence}, pp.\  10--17, 2005.

\bibitem[Audibert et~al.(2010)Audibert, Bubeck, and Munos]{audibert2010best}
Jean-Yves Audibert, S{\'e}bastien Bubeck, and R{\'e}mi Munos.
\newblock Best arm identification in multi-armed bandits.
\newblock In \emph{COLT}, pp.\  41--53, 2010.

\bibitem[Auer et~al.(1995)Auer, Cesa-Bianchi, Freund, and Schapire]{AuerGamblinginaRiggedCasino}
P.~Auer, N.~Cesa-Bianchi, Y.~Freund, and R.~E. Schapire.
\newblock Gambling in a rigged casino: The adversarial multi-armed bandit problem.
\newblock In \emph{Proceedings of the 36th Annual Symposium on Foundations of Computer Science}, FOCS '95, pp.\  322, USA, 1995. IEEE Computer Society.
\newblock ISBN 0818671831.

\bibitem[Auer \& Ortner(2010)Auer and Ortner]{auer2010ucb}
Peter Auer and Ronald Ortner.
\newblock Ucb revisited: Improved regret bounds for the stochastic multi-armed bandit problem.
\newblock \emph{Periodica Mathematica Hungarica}, 61\penalty0 (1-2):\penalty0 55--65, 2010.

\bibitem[Auer et~al.(2002)Auer, Cesa-Bianchi, and Fischer]{auer2002FTAMAB}
Peter Auer, Nicol\`{o} Cesa-Bianchi, and Paul Fischer.
\newblock Finite-time analysis of the multiarmed bandit problem.
\newblock \emph{Mach. Learn.}, 47\penalty0 (2–3):\penalty0 235–256, may 2002.
\newblock ISSN 0885-6125.
\newblock \doi{10.1023/A:1013689704352}.
\newblock URL \url{https://doi.org/10.1023/A:1013689704352}.

\bibitem[Azar et~al.(2017)Azar, Osband, and Munos]{azar2017minimax}
Mohammad~Gheshlaghi Azar, Ian Osband, and Rémi Munos.
\newblock Minimax regret bounds for reinforcement learning, 2017.

\bibitem[Bouneffouf et~al.(2020)Bouneffouf, Rish, and Aggarwal]{bouneffouf2020survey}
Djallel Bouneffouf, Irina Rish, and Charu Aggarwal.
\newblock Survey on applications of multi-armed and contextual bandits.
\newblock In \emph{2020 IEEE Congress on Evolutionary Computation (CEC)}, pp.\  1--8. IEEE, 2020.

\bibitem[Bubeck et~al.(2009)Bubeck, Munos, and Stoltz]{bubeck2009pure}
Sebastien Bubeck, Remi Munos, and Gilles Stoltz.
\newblock Pure exploration in multi-armed bandits problems.
\newblock In \emph{Algorithmic Learning Theory: 20th International Conference, ALT 2009, Porto, Portugal, October 3-5, 2009. Proceedings 20}, pp.\  23--37. Springer, 2009.

\bibitem[Bubeck et~al.(2012)Bubeck, Cesa-Bianchi, et~al.]{bubeck2012regret}
S{\'e}bastien Bubeck, Nicolo Cesa-Bianchi, et~al.
\newblock Regret analysis of stochastic and nonstochastic multi-armed bandit problems.
\newblock \emph{Foundations and Trends{\textregistered} in Machine Learning}, 5\penalty0 (1):\penalty0 1--122, 2012.

\bibitem[Canonne et~al.(2017)Canonne, Diakonikolas, Kane, and Stewart]{canonne2017testing}
Cl{\'e}ment~L Canonne, Ilias Diakonikolas, Daniel~M Kane, and Alistair Stewart.
\newblock Testing bayesian networks.
\newblock In \emph{Conference on Learning Theory}, pp.\  370--448. PMLR, 2017.

\bibitem[Cover \& Thomas(2006)Cover and Thomas]{JoyCoverBook}
Thomas~M. Cover and Joy~A. Thomas.
\newblock \emph{Elements of Information Theory (Wiley Series in Telecommunications and Signal Processing)}.
\newblock Wiley-Interscience, USA, 2006.
\newblock ISBN 0471241954.

\bibitem[Daskalakis et~al.(2019)Daskalakis, Dikkala, and Kamath]{daskalakis2019testing}
Constantinos Daskalakis, Nishanth Dikkala, and Gautam Kamath.
\newblock Testing ising models.
\newblock \emph{IEEE Transactions on Information Theory}, 65\penalty0 (11):\penalty0 6829--6852, 2019.

\bibitem[Devroye(1983)]{Devroye1983}
Luc Devroye.
\newblock The equivalence of weak, strong and complete convergence in l1 for kernel density estimates.
\newblock \emph{The Annals of Statistics}, 11\penalty0 (3):\penalty0 896--904, 1983.
\newblock ISSN 00905364.
\newblock URL \url{http://www.jstor.org/stable/2240651}.

\bibitem[Diamond \& Boyd(2016)Diamond and Boyd]{CVXOPT}
Steven Diamond and Stephen Boyd.
\newblock {CVXPY}: {A} {P}ython-embedded modeling language for convex optimization.
\newblock \emph{Journal of Machine Learning Research}, 17\penalty0 (83):\penalty0 1--5, 2016.

\bibitem[Eberhardt et~al.(2005)Eberhardt, Glymour, and Scheines]{eberhardt2005number}
Frederick Eberhardt, Clark Glymour, and Richard Scheines.
\newblock On the number of experiments sufficient and in the worst case necessary to identify all causal relations among n variables.
\newblock In \emph{Proceedings of the Twenty-First Conference on Uncertainty in Artificial Intelligence}, pp.\  178--184, 2005.

\bibitem[Even-Dar et~al.(2006)Even-Dar, Mannor, Mansour, and Mahadevan]{even2006action}
Eyal Even-Dar, Shie Mannor, Yishay Mansour, and Sridhar Mahadevan.
\newblock Action elimination and stopping conditions for the multi-armed bandit and reinforcement learning problems.
\newblock \emph{Journal of machine learning research}, 7\penalty0 (6), 2006.

\bibitem[Feng \& Chen(2023)Feng and Chen]{Feng_Chen_2023}
Shi Feng and Wei Chen.
\newblock Combinatorial causal bandits.
\newblock \emph{Proceedings of the AAAI Conference on Artificial Intelligence}, 37\penalty0 (6):\penalty0 7550--7558, Jun. 2023.
\newblock \doi{10.1609/aaai.v37i6.25917}.
\newblock URL \url{https://ojs.aaai.org/index.php/AAAI/article/view/25917}.

\bibitem[Hauser \& B{\"u}hlmann(2014)Hauser and B{\"u}hlmann]{hauser2014two}
Alain Hauser and Peter B{\"u}hlmann.
\newblock Two optimal strategies for active learning of causal models from interventional data.
\newblock \emph{International Journal of Approximate Reasoning}, 55\penalty0 (4):\penalty0 926--939, 2014.

\bibitem[Jamshidi et~al.(2024)Jamshidi, Etesami, and Kiyavash]{jamshidi2024confounded}
Fateme Jamshidi, Jalal Etesami, and Negar Kiyavash.
\newblock Confounded budgeted causal bandits.
\newblock \emph{arXiv preprint arXiv:2401.07578}, 2024.

\bibitem[Kang \& Tian(2006)Kang and Tian]{kang2006inequality}
Changsung Kang and Jin Tian.
\newblock Inequality constraints in causal models with hidden variables.
\newblock In \emph{Proceedings of the Twenty-Second Conference on Uncertainty in Artificial Intelligence}, UAI'06, pp.\  233–240, Arlington, Virginia, USA, 2006. AUAI Press.
\newblock ISBN 0974903922.

\bibitem[Kocaoglu et~al.(2017{\natexlab{a}})Kocaoglu, Dimakis, and Vishwanath]{kocaoglu2017cost}
Murat Kocaoglu, Alex Dimakis, and Sriram Vishwanath.
\newblock Cost-optimal learning of causal graphs.
\newblock In \emph{International Conference on Machine Learning}, pp.\  1875--1884. PMLR, 2017{\natexlab{a}}.

\bibitem[Kocaoglu et~al.(2017{\natexlab{b}})Kocaoglu, Shanmugam, and Bareinboim]{kocaoglu2017experimental}
Murat Kocaoglu, Karthikeyan Shanmugam, and Elias Bareinboim.
\newblock Experimental design for learning causal graphs with latent variables.
\newblock \emph{Advances in Neural Information Processing Systems}, 30, 2017{\natexlab{b}}.

\bibitem[Kuleshov \& Precup(2014)Kuleshov and Precup]{kuleshov2014algorithms}
Volodymyr Kuleshov and Doina Precup.
\newblock Algorithms for multi-armed bandit problems.
\newblock \emph{arXiv preprint arXiv:1402.6028}, 2014.

\bibitem[Lattimore et~al.(2016)Lattimore, Lattimore, and Reid]{Lattimore}
Finnian Lattimore, Tor Lattimore, and Mark~D. Reid.
\newblock Causal bandits: Learning good interventions via causal inference.
\newblock In \emph{Proceedings of the 30th International Conference on Neural Information Processing Systems}, NIPS'16, pp.\  1189–1197, Red Hook, NY, USA, 2016. Curran Associates Inc.
\newblock ISBN 9781510838819.

\bibitem[Lattimore \& Szepesv{\'a}ri(2020)Lattimore and Szepesv{\'a}ri]{lattimore2020bandit}
Tor Lattimore and Csaba Szepesv{\'a}ri.
\newblock \emph{Bandit algorithms}.
\newblock Cambridge University Press, 2020.

\bibitem[Lattimore \& Szepesvári(2020)Lattimore and Szepesvári]{LattimoreBook}
Tor Lattimore and Csaba Szepesvári.
\newblock \emph{Bandit Algorithms}.
\newblock Cambridge University Press, 2020.
\newblock \doi{10.1017/9781108571401}.

\bibitem[Lee \& Bareinboim(2018)Lee and Bareinboim]{LeeBarenBoim2018}
Sanghack Lee and Elias Bareinboim.
\newblock Structural causal bandits: Where to intervene?
\newblock In S.~Bengio, H.~Wallach, H.~Larochelle, K.~Grauman, N.~Cesa-Bianchi, and R.~Garnett (eds.), \emph{Advances in Neural Information Processing Systems}, volume~31. Curran Associates, Inc., 2018.
\newblock URL \url{https://proceedings.neurips.cc/paper/2018/file/c0a271bc0ecb776a094786474322cb82-Paper.pdf}.

\bibitem[Lee \& Bareinboim(2019)Lee and Bareinboim]{LeeBareinboim2019}
Sanghack Lee and Elias Bareinboim.
\newblock Structural causal bandits with non-manipulable variables.
\newblock \emph{Proceedings of the AAAI Conference on Artificial Intelligence}, 33\penalty0 (01):\penalty0 4164--4172, Jul. 2019.
\newblock \doi{10.1609/aaai.v33i01.33014164}.
\newblock URL \url{https://ojs.aaai.org/index.php/AAAI/article/view/4320}.

\bibitem[Lu et~al.(2020)Lu, Meisami, Tewari, and Yan]{lu2020regret}
Yangyi Lu, Amirhossein Meisami, Ambuj Tewari, and William Yan.
\newblock Regret analysis of bandit problems with causal background knowledge.
\newblock In \emph{Conference on Uncertainty in Artificial Intelligence}, pp.\  141--150. PMLR, 2020.

\bibitem[Lu et~al.(2021)Lu, Meisami, and Tewari]{lu2021causal}
Yangyi Lu, Amirhossein Meisami, and Ambuj Tewari.
\newblock Causal bandits with unknown graph structure.
\newblock \emph{Advances in Neural Information Processing Systems}, 34:\penalty0 24817--24828, 2021.

\bibitem[Lu et~al.(2022)Lu, Meisami, and Tewari]{lu2022efficient}
Yangyi Lu, Amirhossein Meisami, and Ambuj Tewari.
\newblock Efficient reinforcement learning with prior causal knowledge.
\newblock In \emph{Conference on Causal Learning and Reasoning}, pp.\  526--541. PMLR, 2022.

\bibitem[Maiti et~al.(2022)Maiti, Nair, and Sinha]{maiti2022causal}
Aurghya Maiti, Vineet Nair, and Gaurav Sinha.
\newblock A causal bandit approach to learning good atomic interventions in presence of unobserved confounders.
\newblock In \emph{Uncertainty in Artificial Intelligence}, pp.\  1328--1338. PMLR, 2022.

\bibitem[Malek et~al.(2023)Malek, Aglietti, and Chiappa]{malek2023additive}
Alan Malek, Virginia Aglietti, and Silvia Chiappa.
\newblock Additive causal bandits with unknown graph.
\newblock In \emph{International Conference on Machine Learning}, pp.\  23574--23589. PMLR, 2023.

\bibitem[Mannor \& Tsitsiklis(2004)Mannor and Tsitsiklis]{mannor2004sample}
Shie Mannor and John~N Tsitsiklis.
\newblock The sample complexity of exploration in the multi-armed bandit problem.
\newblock \emph{Journal of Machine Learning Research}, 5\penalty0 (Jun):\penalty0 623--648, 2004.

\bibitem[Nair et~al.(2021)Nair, Patil, and Sinha]{Gaurav2020}
Vineet Nair, Vishakha Patil, and Gaurav Sinha.
\newblock Budgeted and non-budgeted causal bandits.
\newblock In Arindam Banerjee and Kenji Fukumizu (eds.), \emph{The 24th International Conference on Artificial Intelligence and Statistics, {AISTATS} 2021, April 13-15, 2021, Virtual Event}, volume 130 of \emph{Proceedings of Machine Learning Research}, pp.\  2017--2025. {PMLR}, 2021.
\newblock URL \url{http://proceedings.mlr.press/v130/nair21a.html}.

\bibitem[Orabona et~al.(2012)Orabona, Cesa-Bianchi, and Gentile]{orabona2012beyond}
Francesco Orabona, Nicolo Cesa-Bianchi, and Claudio Gentile.
\newblock Beyond logarithmic bounds in online learning.
\newblock In \emph{Artificial intelligence and statistics}, pp.\  823--831. PMLR, 2012.

\bibitem[Osband \& Van~Roy(2016)Osband and Van~Roy]{osband2016lower}
Ian Osband and Benjamin Van~Roy.
\newblock On lower bounds for regret in reinforcement learning.
\newblock \emph{arXiv preprint arXiv:1608.02732}, 2016.

\bibitem[Pearl \& Verma(1995)Pearl and Verma]{pearl1995theory}
Judea Pearl and Thomas~S Verma.
\newblock A theory of inferred causation.
\newblock In \emph{Studies in Logic and the Foundations of Mathematics}, volume 134, pp.\  789--811. Elsevier, 1995.

\bibitem[Peters et~al.(2017)Peters, Janzing, and Schlkopf]{PetersBook}
Jonas Peters, Dominik Janzing, and Bernhard Schlkopf.
\newblock \emph{Elements of Causal Inference: Foundations and Learning Algorithms}.
\newblock The MIT Press, 2017.
\newblock ISBN 0262037319.

\bibitem[Sen et~al.(2017{\natexlab{a}})Sen, Shanmugam, Dimakis, and Shakkottai]{sen2017identifying}
Rajat Sen, Karthikeyan Shanmugam, Alexandros~G Dimakis, and Sanjay Shakkottai.
\newblock Identifying best interventions through online importance sampling.
\newblock In \emph{International Conference on Machine Learning}, pp.\  3057--3066. PMLR, 2017{\natexlab{a}}.

\bibitem[Sen et~al.(2017{\natexlab{b}})Sen, Shanmugam, Kocaoglu, Dimakis, and Shakkottai]{sen2017contextual}
Rajat Sen, Karthikeyan Shanmugam, Murat Kocaoglu, Alex Dimakis, and Sanjay Shakkottai.
\newblock Contextual bandits with latent confounders: An nmf approach.
\newblock In \emph{Artificial Intelligence and Statistics}, pp.\  518--527. PMLR, 2017{\natexlab{b}}.

\bibitem[Shanmugam et~al.(2015)Shanmugam, Kocaoglu, Dimakis, and Vishwanath]{shanmugam2015learning}
Karthikeyan Shanmugam, Murat Kocaoglu, Alexandros~G Dimakis, and Sriram Vishwanath.
\newblock Learning causal graphs with small interventions.
\newblock \emph{Advances in Neural Information Processing Systems}, 28, 2015.

\bibitem[Slivkins et~al.(2019)]{slivkins2019introduction}
Aleksandrs Slivkins et~al.
\newblock Introduction to multi-armed bandits.
\newblock \emph{Foundations and Trends{\textregistered} in Machine Learning}, 12\penalty0 (1-2):\penalty0 1--286, 2019.

\bibitem[Spirtes et~al.(2000)Spirtes, Glymour, Scheines, and Heckerman]{spirtes2000causation}
Peter Spirtes, Clark~N Glymour, Richard Scheines, and David Heckerman.
\newblock \emph{Causation, prediction, and search}.
\newblock MIT press, 2000.

\bibitem[Subramanian \& Ravindran(2022)Subramanian and Ravindran]{subramanian2022causal}
Chandrasekar Subramanian and Balaraman Ravindran.
\newblock Causal contextual bandits with targeted interventions.
\newblock In \emph{International Conference on Learning Representations}, 2022.

\bibitem[Tian \& Pearl(2002)Tian and Pearl]{tian2002testable}
Jin Tian and Judea Pearl.
\newblock On the testable implications of causal models with hidden variables.
\newblock In \emph{Proceedings of the Eighteenth conference on Uncertainty in artificial intelligence}, pp.\  519--527, 2002.

\bibitem[Varici et~al.(2022)Varici, Shanmugam, Sattigeri, and Tajer]{varici2022causal}
Burak Varici, Karthikeyan Shanmugam, Prasanna Sattigeri, and Ali Tajer.
\newblock Causal bandits for linear structural equation models.
\newblock \emph{arXiv preprint arXiv:2208.12764}, 2022.

\bibitem[Wei et~al.(2024)Wei, Elahi, Ghasemi, and Kocaoglu]{wei2024approximate}
Lai Wei, Muhammad~Qasim Elahi, Mahsa Ghasemi, and Murat Kocaoglu.
\newblock Approximate allocation matching for structural causal bandits with unobserved confounders.
\newblock \emph{Advances in Neural Information Processing Systems}, 36, 2024.

\bibitem[Xiong \& Chen(2023)Xiong and Chen]{xiong2023combinatorial}
Nuoya Xiong and Wei Chen.
\newblock Combinatorial pure exploration of causal bandits.
\newblock In \emph{International Conference on Learning Representations}, 2023.

\bibitem[Yabe et~al.(2018)Yabe, Hatano, Sumita, Ito, Kakimura, Fukunaga, and Kawarabayashi]{yabe2018causal}
Akihiro Yabe, Daisuke Hatano, Hanna Sumita, Shinji Ito, Naonori Kakimura, Takuro Fukunaga, and Ken-ichi Kawarabayashi.
\newblock Causal bandits with propagating inference.
\newblock In \emph{International Conference on Machine Learning}, pp.\  5512--5520. PMLR, 2018.

\bibitem[Yan et~al.(2024)Yan, Wei, Katz-Rogozhnikov, Sattigeri, and Tajer]{yan2024causal}
Zirui Yan, Dennis Wei, Dmitriy Katz-Rogozhnikov, Prasanna Sattigeri, and Ali Tajer.
\newblock Causal bandits with general causal models and interventions.
\newblock \emph{arXiv preprint arXiv:2403.00233}, 2024.

\bibitem[Zhang(2008)]{zhang2008completeness}
Jiji Zhang.
\newblock On the completeness of orientation rules for causal discovery in the presence of latent confounders and selection bias.
\newblock \emph{Artificial Intelligence}, 172\penalty0 (16-17):\penalty0 1873--1896, 2008.

\bibitem[Zhang(2020)]{pmlr-v119-zhang20a}
Junzhe Zhang.
\newblock Designing optimal dynamic treatment regimes: A causal reinforcement learning approach.
\newblock In Hal~Daumé III and Aarti Singh (eds.), \emph{Proceedings of the 37th International Conference on Machine Learning}, volume 119 of \emph{Proceedings of Machine Learning Research}, pp.\  11012--11022. PMLR, 13--18 Jul 2020.
\newblock URL \url{https://proceedings.mlr.press/v119/zhang20a.html}.

\end{thebibliography}
